%% file: main.tex
\documentclass{article}
\usepackage[utf8]{inputenc}
\usepackage[T1]{fontenc}

\usepackage{arxiv}

\usepackage{amsmath}
\usepackage{amssymb}
\usepackage{amsfonts}
\usepackage{mathtools}
\usepackage{amsthm}
\usepackage{bbm}

\usepackage{microtype}
\usepackage{graphicx}
\usepackage{subcaption}
\usepackage{enumitem}
\usepackage{multirow}
\usepackage{multicol}
\usepackage{comment}
\usepackage{booktabs}

\usepackage{nicefrac}
\usepackage{url}
\usepackage{natbib}
\usepackage{doi}

\usepackage{hyperref}
\usepackage[capitalize,noabbrev]{cleveref}

\theoremstyle{plain}
\newtheorem{theorem}{Theorem}[section]
\newtheorem{proposition}[theorem]{Proposition}
\newtheorem{lemma}[theorem]{Lemma}

\theoremstyle{definition}

\theoremstyle{remark}

\title{The Energy Consumption of Transformer Fine-Tuning: A Roofline-Inspired Scaling Model}

\author{
    Mansour Zoubeirou a Mayaki \\
    \small{Université Lumière Lyon 2, CNRS, Ecole Centrale de Lyon, INSA Lyon,} \\
    \small{Universite Claude Bernard Lyon 1, LIRIS, UMR5205, 69007 Lyon, France} \\
    \texttt{mansour.mayaki at liris.cnrs.fr}
}

\renewcommand{\headeright}{}
\renewcommand{\undertitle}{}
\renewcommand{\shorttitle}{The Energy Consumption of Transformer Fine-Tuning}

\hypersetup{
pdftitle={The Energy Consumption of Transformer Fine-Tuning: A Roofline-Inspired Scaling Model},
pdfauthor={Mansour Zoubeirou a Mayaki},
pdfkeywords={transformer fine-tuning, energy consumption, roofline model, scaling laws}
}

\begin{document}

\maketitle
\begin{abstract}
\input{abstract}
\end{abstract}
\input{intro}
\input{revue}
\input{methodo}
\input{preModel}
\input{results}
\input{validation}
\input{Discussion}
\input{conclusion}

\clearpage
\appendix
\onecolumn

\input{results_appendix}
\input{proofs}
\input{modelConfigs}

\bibliographystyle{icml2026}
\bibliography{references}

\end{document}

%% file: abstract.tex
Transformer-based models underpin modern natural language processing but incur rapidly growing computational and energy costs. As training scales in both model size and parallelism, accurately predicting energy consumption has become critical for sustainable and cost-aware system design. We present a framework for modeling the energy consumption of Transformer training on multiple GPUs. Using controlled architectural
sweeps of BERT models, we relate measured energy to lightweight proxies for compute, memory traffic, and hardware efficiency. Inspired by roofline models, our approach incorporates a speedup-based hardware-efficiency factor that captures the effects of tensor parallelism and fully sharded data parallelism. We derive a scaling law model that accurately predicts training energy across heterogeneous configurations.

%% file: intro.tex
\section{Introduction}

The rapid growth of deep learning has driven remarkable progress across natural
language processing (NLP), computer vision, and multimodal reasoning. In
particular, Transformer-based architectures have become the dominant paradigm for
language modeling and representation learning, underpinning widely deployed models
such as BERT, T5, and GPT. This progress, however, has come at the cost of rapidly
increasing computational and energy demands. Training modern Transformer models
requires substantial compute resources, leading to high operational costs and a
growing environmental footprint.

Recent work has raised serious concerns about the sustainability of large-scale
machine learning \cite{wu2022sustainable,mehlin2023towards,yarally2023uncovering}.
As models scale in depth, width, and training data, energy consumption grows
nonlinearly and becomes increasingly sensitive to architectural choices,
parallelization strategies, and hardware efficiency. While several studies report
aggregate energy or carbon footprints for individual large models, a systematic
understanding of how \emph{training energy scales} with compute, memory traffic, and
parallel efficiency remains limited. Existing approaches often conflate model size
with hardware utilization and provide little guidance for predicting energy
consumption under new architectures or training configurations.
This paper addresses this gap by developing a predictive framework for training
energy consumption in Transformer models. Rather than treating energy as a
black-box outcome, we decompose it into interpretable components driven by
(i) total compute, (ii) memory traffic, and (iii) hardware efficiency. Our approach
is inspired by classical roofline models from high-performance computing and is
augmented with an empirically grounded treatment of parallel efficiency derived
from execution-time speedup.

We conduct controlled architectural sweeps of BERT models using the Hugging Face
implementation, varying depth, width, feed-forward dimension, and attention
configuration. From these sweeps, we construct lightweight proxies for parameter
count, floating-point operations, and activation memory traffic. To capture
parallelization effects, we introduce a speedup-based hardware-efficiency factor
for tensor parallelism and fully sharded data parallelism.

This work focuses on \emph{operational training energy} under controlled
architectural scaling, with empirical evaluation on BERT-style fine-tuning
workloads and common distributed training strategies (data parallelism, FSDP, and
tensor parallelism). Our goal is not to provide a universal carbon or energy
accounting framework, but to identify the dominant scaling mechanisms governing
training energy once system-level efficiency is taken into account. While our
experiments are conducted on encoder-based Transformers, the proposed
roofline-inspired formulation and speedup-based efficiency factor rely only on
generic compute--memory interactions and execution-time scaling, and therefore
extend conceptually to other Transformer architectures. Empirical generalization
beyond the studied regime is intentionally framed as trend-level consistency rather
than direct validation.

\vspace{-0.3cm}
\paragraph{Contributions.}
The main contributions of this paper are:
\vspace{-0.3cm}
\begin{itemize}[leftmargin=*]
    \item We propose a mechanistic, roofline-inspired decomposition of Transformer
    training energy into compute, memory, and hardware-efficiency components.
    \item We introduce a speedup-based hardware-efficiency model for tensor and data parallelism, enabling energy estimation without direct
    multi-GPU energy measurements.
    \item We derive and empirically validate a scaling law that predicts training
    energy across a wide range of BERT configurations.
\end{itemize}

%% file: revue.tex
\section{Related Work}

We review three strands of work most relevant to this paper:
(i) measurement and reporting of energy and carbon in machine learning,
(ii) empirical evidence on how scale, architecture, and parallelization affect
efficiency, and (iii) roofline-style frameworks for reasoning about compute--memory
trade-offs. We conclude by positioning our contribution relative to these lines of
work.

\subsection{Measuring and Reporting Energy and Carbon}
Early empirical studies emphasized the need for standardized reporting of energy
consumption and carbon emissions in machine learning, highlighting substantial
variation across measurement tools and experimental protocols
\cite{henderson2020towards}. Subsequent analyses documented the rapid growth of
energy use in state-of-the-art training and the importance of system-level factors
such as cooling, power usage effectiveness (PUE), and geographic siting
\cite{cottier2024rising,luccioni2024power}. Methodological work has proposed
approaches for estimating emissions from public artifacts
\cite{luccioni2023estimating}, as well as telemetry pipelines that attribute energy
across CPU, GPU, and memory subsystems
\cite{yarally2023uncovering,10244360}. In NLP, \citet{strubell2020energy}
popularized reporting carbon alongside accuracy, while more recent lifecycle-aware
frameworks combine operational and embodied emissions to support forward estimation
under varying hardware and grid assumptions \cite{faiz2024llmcarbon}.

\subsection{Scale, Architecture, and Efficiency}
Classical scaling laws relate parameter count, data, and compute to model quality
\cite{Kaplan2020,brown2020language}, but they do not directly characterize training
energy, which is mediated by utilization, memory traffic, and parallel efficiency.
Systems studies have shown that realized throughput for large models depends
critically on parallelization strategy and communication behavior
\cite{narayanan2021efficient}. Complementary case studies further show that
configuration choices such as batching, numerical precision, and interconnect
topology can induce large differences in energy consumption for comparable training
objectives \cite{patterson2021carbon}. Recent work has also begun to study energy scaling more directly, including efficiency-oriented analyses of local LLM deployment and operation
\cite{alvarez2025scalinglawsenergyefficiency}. In particular, \citet{alvarez2025scalinglawsenergyefficiency}
study CPU-only local inference for LLMs and VLMs on consumer and embedded devices,
showing that inference cost scales approximately linearly with token length for
LLMs and exhibits preprocessing-dependent threshold effects for VLMs. Their results
highlight compression and input preprocessing as practical levers for sustainable
edge inference, whereas our work focuses on \emph{distributed training}, where
parallelization strategy and execution efficiency are the primary determinants of
energy scaling.

\subsection{Roofline Models}
Roofline models from high-performance computing provide a principled framework for
reasoning about the interaction between compute throughput, memory bandwidth, and
achieved performance \cite{williams2009roofline}. Extensions of the roofline
approach have been used to analyze GPU kernels and deep learning workloads,
highlighting the roles of arithmetic intensity, memory traffic, and utilization in
determining performance and energy efficiency
\cite{6569852,jouppi2017datacenter,9297122,owens2018gpu}. More recent studies apply
roofline-style analyses to neural network training to explain when workloads are
compute-bound versus memory-bound and how architectural choices shift this balance
\cite{JacekApplyingRoofline,owens2018survey}. However, most roofline-based analyses
focus on performance or power at the kernel or single-device level, rather than on
training energy under distributed parallelism.

\subsection{Positioning of This Work}

Most prior work falls into three categories:
(i) reporting aggregate energy or emissions for specific training runs
\cite{strubell2020energy,patterson2021carbon},
(ii) modeling end-to-end emissions without explicitly isolating execution
efficiency \cite{faiz2024llmcarbon}, or
(iii) studying scaling and parallelization primarily through throughput and
time-to-solution rather than energy \cite{Kaplan2020,brown2020language}.
Recent work by \citet{zoubeirouamayaki2025energy} proposes an empirical energy
model based on power-law relationships among energy, FLOPs, and a
hardware-efficiency factor, validated on single-GPU measurements across LSTM, GRU,
and Transformer models. Their approach decomposes Transformer layers into
elementary operations and fits operation-level efficiency curves per GPU, yielding
accurate predictions under fixed hardware and execution conditions.

Our work is complementary but distinct. We target \emph{distributed Transformer
training} and explicitly study how training energy scales with parallelization.
Drawing inspiration from roofline theory, we decompose training energy into
compute- and memory-related components, then augment this decomposition with a
speedup-derived efficiency proxy that captures the effects of multi-GPU execution.
The resulting model is not intended as a universal energy law across all
architectures and hardware platforms; rather, it provides an interpretable and
predictive framework for reasoning about training energy within a controlled regime,
with explicit attention to tensor parallelism and fully sharded data parallelism.

%% file: methodo.tex
\section{Methodology}
\label{sec:method}

This section details our framework for predicting BERT \cite{devlin2019bert}
fine-tuning energy. Our methodology combines controlled architectural sweeps with
closed-form proxies for parameter count, compute, and memory traffic, together with
an explicit efficiency term for parallel execution. The goal is not to causally
isolate perfectly orthogonal factors, but to obtain sufficient variation in
architecture and system-level quantities to support an interpretable regression
model of training energy.

\textbf{Model map.}
Eq.~\eqref{eq:energy-decomp-main} gives a mechanistic decomposition of training
energy into compute, memory, and efficiency terms; Eq.~\eqref{eq:reduced-form-main}
gives a reduced-form scaling law fitted to these components.

\subsection{Tasks, Models and Training Protocol}
\label{subsec:tasks-models}

\textbf{Tasks and model suite.}
We fine-tune a family of BERT models on a suite of standard NLP tasks, including
sentence classification and extractive question answering. These models are
implemented using the Hugging Face \texttt{transformers} library, specifically
\texttt{BertForSequenceClassification}. To
study scaling behavior, we vary four core architectural hyperparameters: the number
of Transformer layers $L$, hidden dimension $d$, number of attention heads $h$, and
feed-forward dimension.

\textbf{Controlled sweeps.}
The architectural sweeps are designed to provide broad \emph{coverage} of the
configuration space rather than strict orthogonal control of compute and memory.
Changing depth, width, sequence length, and batch size necessarily affects multiple
quantities simultaneously; this is inherent to Transformer scaling. Our aim is
therefore to generate sufficient variation across configurations to statistically
disentangle the contributions of compute, memory proxy, and execution efficiency in
the downstream regression model.
Unless otherwise stated, each configuration is fine-tuned for a fixed number of
epochs under a consistent optimizer and learning-rate schedule. We log step-level
telemetry (time, power/energy counters when available) and aggregate metrics within
and across epochs. Full training hyperparameters are reported in
Appendix~\ref{app:train-details}.

\subsection{Compute, Parameter and Memory Proxies}
\label{subsec:proxies}

To connect architectural choices to energy, we use closed-form proxies for model
size, compute, and activation/memory traffic
\cite{devlin2019bert,brown2020language}. These proxies are lightweight, depend only
on $(L,d,\mathrm{ff})$ and sequence length, and are intended as \emph{minimal
sufficient statistics} for comparative scaling analysis rather than exact kernel- or
byte-level measurements.

\textbf{Parameter proxy.}
We approximate the dominant non-embedding parameter terms per Transformer layer
(attention projections and MLP weights), yielding
\begin{equation}
\label{eq:bert_params_proxy}
N(L,d,\mathrm{ff})
\;\approx\;
L\left(12d^2 + 2d\,\mathrm{ff}\right).
\end{equation}

\textbf{Compute proxy (FLOPs per token).}
We approximate training compute per token using a simplified decomposition into
projection/MLP cost and attention-related cost, giving
\begin{equation}
\label{eq:bert_flops_proxy}
C_{\text{tok}}(L,d,S)
\;\approx\;
12L d^2 \;+\; 3L S d,
\end{equation}
where $S$ denotes sequence length. Total training compute over a run is then
$C \approx C_{\text{tok}} \times (\text{tokens processed})$.

\textbf{Activation/memory traffic proxy.}
We use a proxy proportional to activation volume across layers,
\begin{equation}
\label{eq:bert_bytes_proxy}
M_{\mathrm{proxy}}(B,S,d,L)
\;\approx\;
B\,S\,d\,L,
\end{equation}
where $B$ is batch size. This proxy captures the leading dependence of
activation/gradient traffic on batch size, sequence length, width, and depth, and
serves as a practical substitute for direct bandwidth measurements. We emphasize
that $M_{\mathrm{proxy}}$ is an aggregate architectural proxy: it does not attempt to
separate activation traffic, optimizer-state movement, or communication volume at a
kernel level.

These proxies are used throughout the paper to relate measured energy to
architecture-dependent compute and memory demands under controlled sweeps.

\subsection{Instrumentation and Energy Accounting}
\label{subsec:instrumentation}

We measure training energy using complementary telemetry from Nsight Systems,
CodeCarbon, and system-utilization logs. Measurements are converted to kWh,
temporally aligned, and integrated over the active training interval after
discarding warm-up and cool-down phases. For each configuration, we report mean
energy over 5 epochs with standard error. Additional instrumentation details are
provided in Appendix~\ref{app:train-details}.

\subsection{Compute--Memory Energy Model}
\label{subsec:comp-mem}

We model batch energy per training step as the sum of compute, memory traffic, and
fixed overheads:
\begin{equation}
\label{eq:energy-decomp-main}
E = \epsilon_{\mathrm{comp}}\cdot C + \epsilon_{\mathrm{mem}}\cdot M + E_{0},
\qquad
\epsilon_{\mathrm{comp}}, \epsilon_{\mathrm{mem}} > 0.
\end{equation}
Here, $C$ denotes floating-point operations and $M$ denotes bytes moved to/from
high-bandwidth memory. The constant $E_{0}$ captures host, leakage, and unmodeled
overheads.

\begin{proposition}[Roofline energy lower bounds]
\label{prop:roofline}
Let $C$ be the FLOPs executed in one training step, $M$ the DRAM/HBM bytes moved,
$\mathrm{AI}\triangleq C/M$ the arithmetic intensity, $F_{\mathrm{peak}}$ the device
peak FLOP/s, and let $P_{\mathrm{dyn}}$ denote the dynamic power during the active
portion of the step. Define $\tau_{\mathrm{mem}}$ as the \emph{effective memory
service rate} (bytes/s), i.e., the rate at which memory traffic is serviced under
the workload. The classical roofline argument implies that the achieved throughput
is upper bounded by
\[
\min\{F_{\mathrm{peak}},\, \tau_{\mathrm{mem}}\cdot \mathrm{AI}\},
\]
which characterizes whether computation or memory traffic limits performance and
energy. Then the step time $T$ satisfies
\begin{equation}
T \;\ge\; \frac{C}{\min\{F_{\mathrm{peak}},\, \tau_{\mathrm{mem}}\cdot \mathrm{AI}\}},
\end{equation}
which yields the energy lower bounds
\begin{equation}
E \;\ge\; \frac{P_{\mathrm{dyn}}}{F_{\mathrm{peak}}}\cdot C
\quad\text{and}\quad
E \;\ge\; \frac{P_{\mathrm{dyn}}}{\tau_{\mathrm{mem}}}\cdot M.
\label{eq:roofline-bounds}
\end{equation}
\end{proposition}

Full details for the proof of Proposition~\ref{prop:roofline} can be found in
Appendix~\ref{proof:roofline}.

\begin{theorem}[Expected energy scaling bands]
\label{thm:bands}
Assume (i) quasi-stationary power during a step; (ii) a stable op-mix as
hyperparameters vary; and (iii) over the hyperparameter ranges studied there exist
$k_1,k_2>0$ and an exponent $\beta\in[\tfrac12,1)$ such that the dominant DRAM
traffic co-scales with compute as
\[
k_1\,C^{\beta} \;\le\; M \;\le\; k_2\,C^{\beta}.
\]
Then for device-specific positives $\epsilon_{\mathrm{comp}},
\epsilon_{\mathrm{mem}}$ and a fixed overhead $E_0\ge 0$,
\[
\min\{\epsilon_{\mathrm{mem}}\,M,\;\epsilon_{\mathrm{comp}}\,C\}
\;\lesssim\;
E
\;\lesssim\;
\epsilon_{\mathrm{comp}}\,C+\epsilon_{\mathrm{mem}}\,M+E_0.
\]
\end{theorem}


\begin{proof}[Proof sketch]
See Appendix~\ref{app: energy} for details.
\end{proof}
\vspace{-0.5cm}

\paragraph{Estimated compute--memory energy model.}
While Eq.~\eqref{eq:energy-decomp-main} captures the physical contributions of
compute and memory to energy consumption, the strategy-dependent offsets in
Appendix~\ref{app:comp-mem} indicate that parallelization alters energy beyond what
can be explained by $C$ and $M$ alone. In particular, tensor and sharded data
parallelism introduce systematic energy effects that are not fully absorbed into
per-FLOP or per-byte costs. This motivates an explicit hardware-efficiency factor,
modeled in the next subsection via execution-time speedup.

\input{speedUp}

\subsection{Estimation and Identifiability}
\label{subsec:estimation}

Theoretical bounds from Theorem~\ref{thm:bands} imply that, in regimes where the
baseline offset $E_0$ is negligible, the effective log--log slope of energy with
respect to compute lies between the memory-dominated and compute-dominated limits.
This motivates the reduced-form scaling law
\begin{equation}
\label{eq:reduced-form-main}
E = \kappa \cdot C^{\alpha_C}\, M^{\alpha_M}\, \eta_h^{\alpha_\eta},
\hspace{0.5cm}
\tfrac12 < \alpha_C \le 1,
\quad
\alpha_\eta < 0.
\end{equation}

Here, $C$ denotes total compute (FLOPs), $M$ is a proxy for memory traffic, and
$\eta_h$ is a hardware-efficiency factor derived from empirical speedup. We stress
that $\eta_h$ is an \emph{efficiency proxy}, not a direct physical measurement: it
summarizes the combined effects of communication, synchronization, and execution
inefficiency under a given hardware class, interconnect topology, and parallel
strategy. The reduced-form model can thus be viewed as a compact representation of
the additive compute--memory decomposition in
Eq.~\eqref{eq:energy-decomp-main}, augmented with a multiplicative efficiency
correction.
We estimate both the mechanistic decomposition in
Eq.~\eqref{eq:energy-decomp-main} and the reduced form in
Eq.~\eqref{eq:reduced-form-main} using constrained log-linear regression with
architecture- and strategy-specific fixed effects. We do \emph{not} assume strict
independence between compute, memory, and hardware efficiency; rather, they are
treated as correlated predictors in an empirical scaling model. To assess stability,
we report model diagnostics and robustness analyses in Appendix~\ref{app:diag}, including checks
of linearity and functional form, residual behavior, feature-specific bias, error
structure and normality, influence, and numerical stability.

%% file: speedUp.tex
\subsection{Hardware Efficiency via Empirical Speedup Models}
\label{subsec:hef}

The strategy-dependent energy offsets observed in
Section~\ref{subsec:comp-mem} motivate modeling parallel efficiency explicitly.
Rather than modeling communication bandwidth, collective latency, or microbatch-level
synchronization costs in detail, we represent hardware efficiency through
\emph{execution-time speedup}. This yields a compact, dimensionless proxy that
captures the combined effects of communication, synchronization, kernel
fragmentation, and utilization under a fixed hardware class and operating regime.

Let $T(N)$ denote the execution time per training step (or epoch) when using $N$
GPUs under a fixed workload. The corresponding speedup is
\begin{equation}
S(N) \triangleq \frac{T(1)}{T(N)}.
\end{equation}
Because execution time is often easier to obtain and reproduce than direct
multi-GPU energy measurements, we use $S(N)$ as the basis for an
energy-relevant efficiency proxy. We emphasize that this proxy is not a direct
physical measurement of efficiency; rather, it summarizes how parallel execution
modifies the energy cost of a fixed workload.

\paragraph{Definition.}
We restrict attention to two parallelism modes: \emph{tensor parallelism} (degree
$t$) and \emph{data parallelism} (degree $d$). In our setup, $t$ is the number of
GPUs that jointly compute shards of each layer's linear algebra (typically within a
node), while $d$ is the number of GPUs that process distinct mini-batches and
periodically synchronize gradients. Let $N=t\,d$ denote the total number of devices
and $g$ the number of GPUs per node. We factor
\begin{equation}
\label{eq:speedup}
\phi(t,d)
\;=\;
\underbrace{\eta_{\mathrm{TP}}(t)}_{\text{tensor-parallel efficiency}}
\cdot
\underbrace{\eta_{\mathrm{DP}}(d)}_{\text{data-parallel efficiency}},
\end{equation}
with $t\in\{1,2,4,\ldots,g\}$ and $d\in\{1,2,\ldots,\lfloor g/t\rfloor\}$ in the
single-node regime. Here, $\phi\in(0,1]$ accounts for the effective efficiency loss
induced by tensor-parallel (TP) and data-parallel (DP/FSDP) execution. Any effects
of interconnect topology (e.g., NVLink vs.\ PCIe), synchronization pattern, or
collective implementation are absorbed into the fitted efficiency curves below.

\paragraph{Tensor-parallel efficiency.}
Tensor parallelism introduces per-layer collective communication and can reduce
kernel efficiency as the tensor-parallel degree $t$ increases. Empirically, TP
exhibits rapid initial gains followed by saturation once communication and kernel
fragmentation dominate. We therefore model TP speedup using a saturating exponential
law:
\begin{equation}
\label{eq:exp_decay}
S_{\mathrm{TP}}(t)
=
S_{\max}\bigl(1 - e^{-k t}\bigr),
\end{equation}
where $S_{\max}$ is the asymptotic speedup limit and $k>0$ controls the rate of
saturation. The corresponding tensor-parallel efficiency is
\begin{equation}
\eta_{\mathrm{TP}}(t)
=
\frac{S_{\mathrm{TP}}(t)}{t}.
\end{equation}
This parameterization captures diminishing returns with increasing TP degree
without introducing explicit bandwidth or latency terms.

\paragraph{Data-parallel efficiency.}
Data parallelism incurs gradient synchronization once per global batch, with cost
depending on communication volume and overlap. Following empirical observations from
large-scale distributed training benchmarks, we model DP speedup using a generalized
Amdahl-like law:
\begin{equation}
\label{eq:Amdahl}
S_{\mathrm{DP}}(d)
=
\frac{d}{1 + \alpha \cdot d^{\beta}},
\end{equation}
where $\alpha\ge 0$ captures the effective communication- or serial-bound fraction
and $\beta\ge 0$ allows synchronization overhead to grow nonlinearly with scale.
The corresponding data-parallel efficiency is
\begin{equation}
\eta_{\mathrm{DP}}(d)
=
\frac{S_{\mathrm{DP}}(d)}{d}
=
\frac{1}{1 + \alpha \cdot d^{\beta}}.
\end{equation}
This model captures the observed sublinear scaling of FSDP/DP and explains the
efficiency degradation at larger device counts.

\textbf{Empirical estimation of speedup.}
Figure~\ref{fig:speedup_fits} illustrates the empirical scalability of
data-parallel (FSDP/DP) and tensor-parallel (TP) training as a function of GPU
count \cite{nichols2022survey}. For data parallelism, measured speedup is well
captured by Eq.~\eqref{eq:Amdahl}, with fitted parameters $\alpha=0.005$ and
$\beta=1.266$. The exponent $\beta>1$ indicates that synchronization and
communication overheads grow faster than linearly at scale, leading to diminishing
returns as the number of devices increases.

In contrast, tensor parallelism exhibits rapidly saturating speedup that is well
approximated by Eq.~\eqref{eq:exp_decay}, with $S_{\max}\approx 9.8$ and
$k\approx 0.76$. This reflects strong initial gains from partitioning within-layer
computation, followed by saturation once collective communication and kernel
fragmentation dominate. TP therefore appears most effective for relatively small
device counts, particularly in intra-node settings.

\begin{lemma}[From speedup to energy efficiency]
\label{lem:speedup_to_energy}
Consider a fixed training workload executed on $N$ devices, and let $T(N)$ denote
the corresponding execution time. Let $P(t)$ be the instantaneous power draw and
$E(N)=\int_0^{T(N)} P(t)\,dt$ the total energy consumed. Assume that:
\vspace{-0.4cm}
\begin{enumerate}
    \item[(A1)] The workload is fixed across device counts, i.e., the total amount
    of computation and memory traffic does not change with $N$.
    \item[(A2)] The average power draw
    \[
    \bar{P}(N)\triangleq \frac{1}{T(N)}\int_0^{T(N)}P(t)\,dt
    \]
    varies slowly with $N$ within a fixed hardware class and operating regime.
    \item[(A3)] Changes in execution time dominate changes in average power as $N$
    varies.
\end{enumerate}
Then energy scaling with respect to $N$ is primarily governed by execution-time
scaling, and the parallel speedup $S(N)=T(1)/T(N)$ induces a corresponding
hardware-efficiency proxy
\[
\eta_h(N)\;\propto\;\frac{S(N)}{N}.
\]
In particular, any parametric model that accurately captures $T(N)$ (and hence
$S(N)$) provides a consistent proxy for the dependence of energy consumption on
parallelization within the studied regime.
\end{lemma}
\vspace{-0.5cm}
\begin{proof}[Proof sketch]
See Appendix~\ref{proof:speedup_to_energy} for details.
\end{proof}

\vspace{-0.5cm}

\begin{figure}[t]
\centering
\includegraphics[width=\columnwidth]{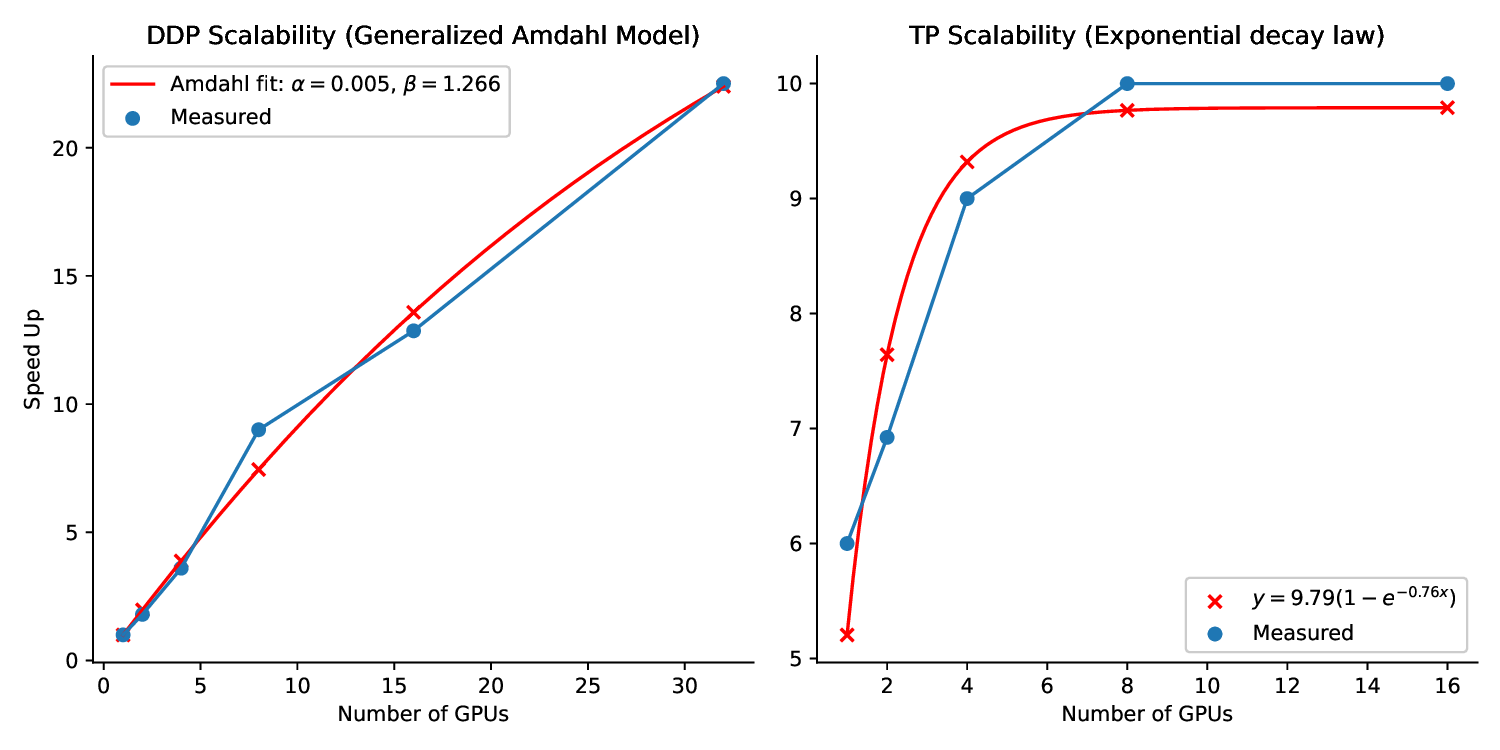}
\vspace{-20pt}
\caption{
Execution-time speedup as a function of GPU count for data-parallel
(FSDP/DP, left) and tensor-parallel (TP, right) training.
Points show measured speedups; curves denote fitted models.
The gap from ideal linear scaling motivates explicit modeling of
parallel efficiency in energy estimation.
}
\label{fig:speedup_fits}
\vspace{-10pt}
\end{figure}

%% file: preModel.tex
\section{Energy Scaling with Hyperparameters}
\label{sec:results}

We evaluate how training energy scales with compute, hardware resources, and
parallelization strategy. The experimental platform consists of a single 8-GPU node with GeForce RTX 2080 Ti accelerators. All energies are averaged over epochs and reported in
kilowatt-hours (kWh).
FLOP counts correspond to a full training step
(forward and backward passes) unless otherwise stated.\\
\textbf{Energy vs.\ compute.}
Figure~\ref{fig:energyVsFlops} (in Appendix~\ref{subsec:energy-flops}) shows that training energy increases monotonically
with total compute (FLOPs) but does so sub-linearly. The large dispersion at fixed
compute confirms that FLOPs alone are insufficient to predict energy consumption.
Across all regimes, fully sharded data parallelism (FSDP) consistently consumes
less energy than tensor parallelism (TP), with the gap widening at higher compute.

\begin{figure}[ht]
    \centering
    \includegraphics[width=\columnwidth]{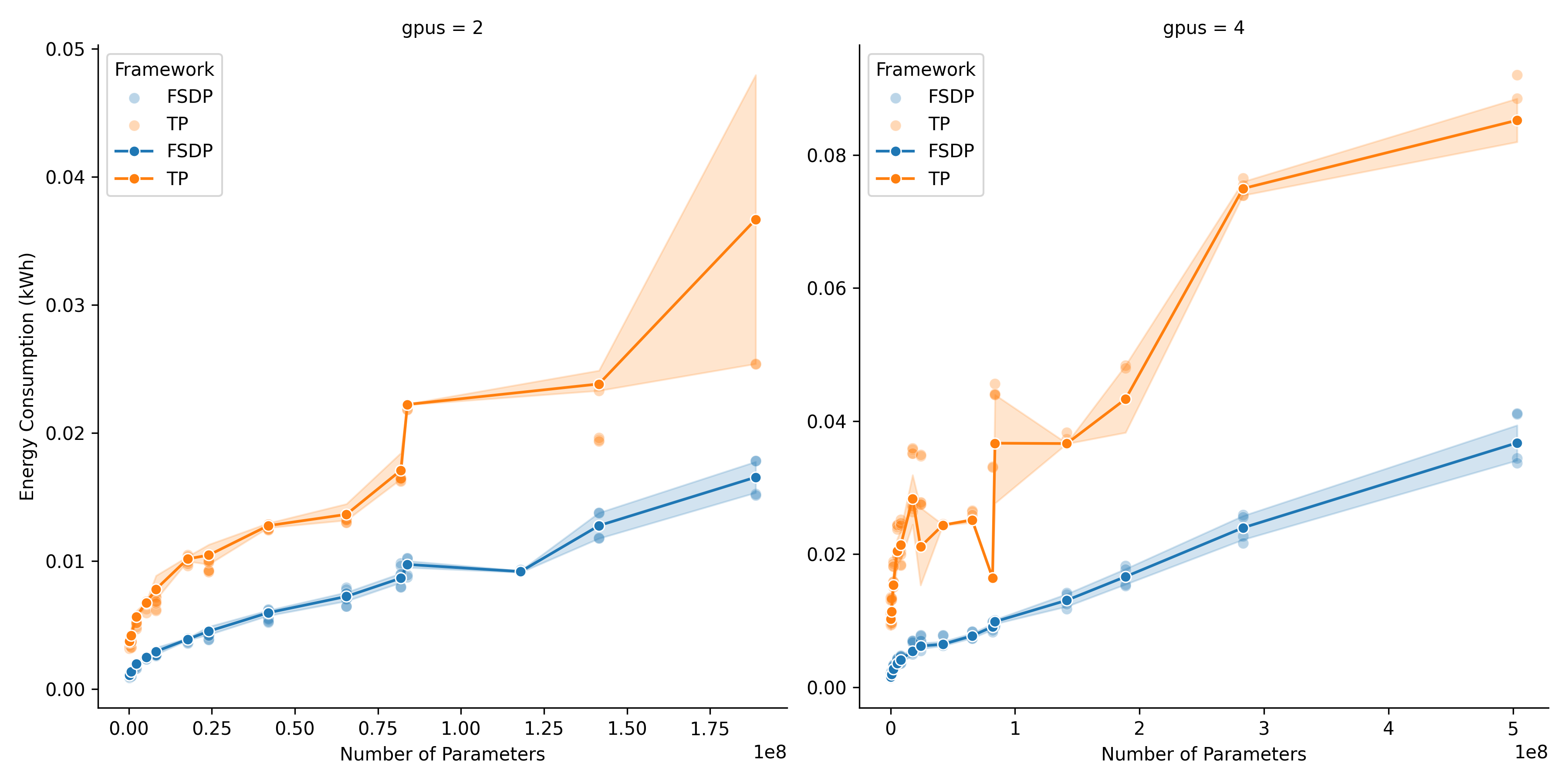}
    \vspace{-20pt}
    \caption{
    Training energy consumption as a function of model size for tensor parallelism (TP) and fully sharded data parallelism (FSDP), shown
    separately for $2$ and $4$ GPUs. 
    }
    \label{fig:energyVsParams}
\end{figure}

\vspace{-10pt}

\textbf{Energy vs.\ model size.}
Figure~\ref{fig:energyVsParams} shows that energy increases with model size
(number of parameters), but the rate depends strongly on the parallelization
strategy. FSDP scales more favorably than TP, particularly for larger models.
Depth-specific analysis are reported in Appendix~\ref{subsec:energy-layers}.

Overall, training energy is governed by a three-way interaction between compute,
parallelization strategy, and hardware efficiency. These results motivate the
energy model introduced in Section~\ref{sec:method}, which explicitly accounts for
hardware efficiency via $\eta_h$.

\textbf{Energy vs.\ number of GPUs.}
As shown in Figure~\ref{fig:energyVsGpus}, increasing the number of GPUs raises total
energy consumption despite reducing wall-clock time. This reflects higher
instantaneous power draw and communication overheads that are not fully offset by
speedup. FSDP achieves lower energy than TP at the same device count, while TP
exhibits greater variance.
\begin{figure}[ht]
    \centering
    \includegraphics[width=\columnwidth]{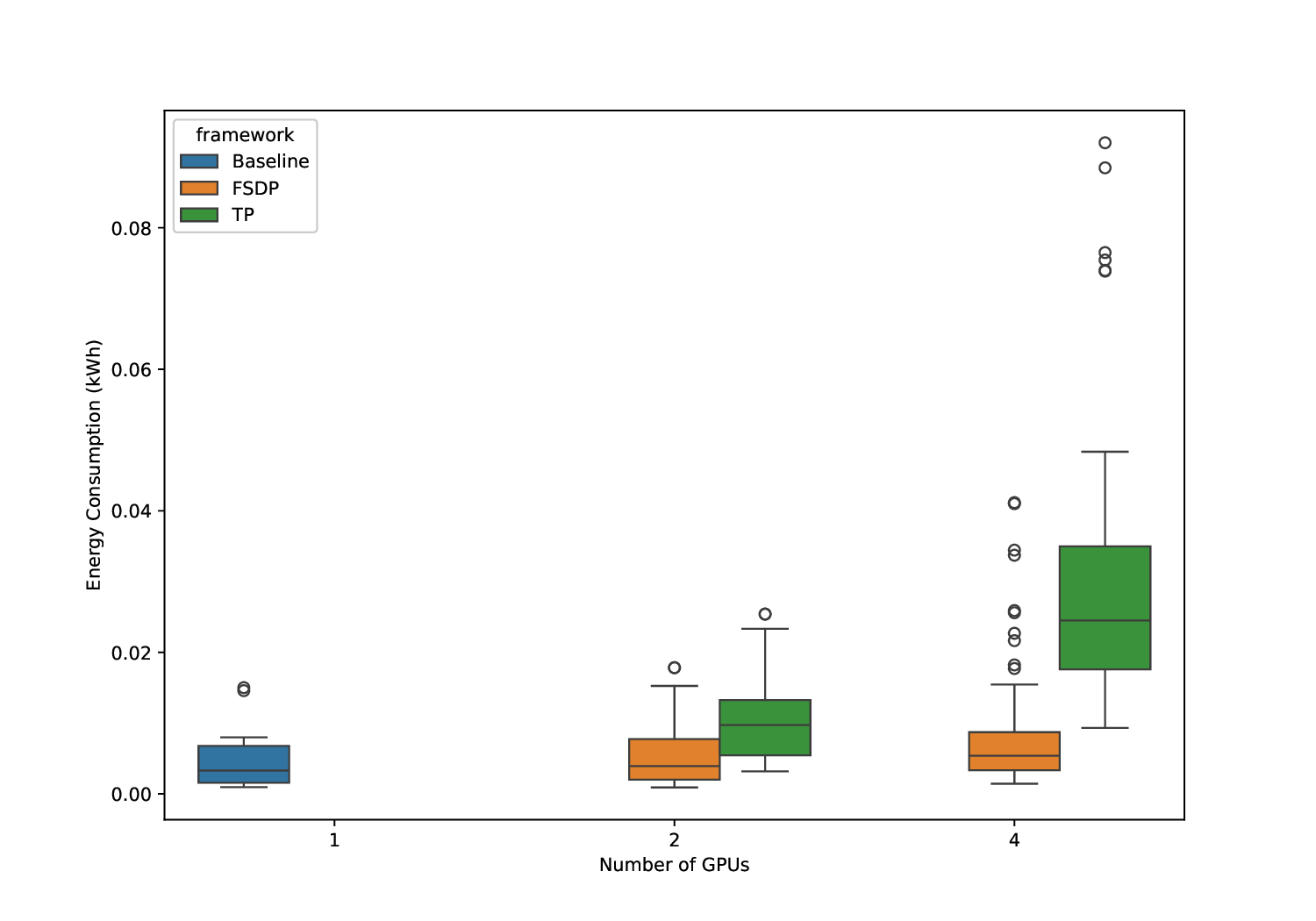}
    \vspace{-20pt}
    \caption{Training energy consumption as a function of the number of GPUs, shown for the single-GPU baseline, fully sharded data parallelism
    (FSDP) and tensor parallelism (TP). 
    }
    \label{fig:energyVsGpus}
    \vspace{-10pt}
\end{figure}

\textbf{Energy vs.\ training duration.}
Figure~\ref{fig:energyVsDuration} in Appendix highlights a clear decoupling between runtime and
energy. Faster configurations often consume more energy, demonstrating that
minimizing time-to-solution does not necessarily minimize energy usage. This
motivates explicitly modeling hardware efficiency rather than relying on runtime
as a proxy.

%% file: results.tex
\section{Energy Model Estimation}
\label{subsec:estimation_results}

\subsection{Implementation Details}
All regressions are performed in log space using ordinary least squares with
heteroskedasticity-robust (HC3) standard errors. The reduced-form model in
Eq.~\eqref{eq:reduced-form-main} is estimated with explicit regressors for compute,
memory traffic, and hardware efficiency. Parallelization strategy effects are
modeled using categorical fixed effects, with single-GPU baseline runs serving as
the reference category.
\vspace{-0.4cm}

\paragraph{Train--validation protocol.}
Configurations are stratified by architecture and randomly partitioned into a
70/30 train--validation split. This procedure is repeated over five random seeds,
and we report medians and interquartile ranges (IQR) for all metrics. Model fitting
is performed exclusively on the training subset; validation results are strictly
out of sample.
\vspace{-0.4cm}

\paragraph{Primary metrics.}
We report (i) out-of-sample $R^2$ ($R^2$), (ii) RMSE, and (iii) calibration slope and intercept obtained by regressing observed energy on predicted energy in the validation set. We further
compute 95\% prediction intervals using a residual bootstrap on the training data
and evaluate empirical coverage on held-out configurations.

\subsection{Final Estimated Energy Model}
\label{subsec:final-model}

Unless otherwise stated, all reported results use the reduced-form model
Eq.~\eqref{eq:reduced-form-main}, with the hardware-efficiency term $\eta_h$
parameterized via empirical speedup models (Eq.~\eqref{eq:speedup}). Estimated
coefficients and confidence intervals are summarized in
Table~\ref{tab:exponents} in Appendix~\ref{app:final-model}. The model achieves an adjusted coefficient of
determination $R^2_{\mathrm{adj}} = 0.853$ over 410 independent configurations,
indicating strong explanatory power across heterogeneous architectures and
parallelization strategies.
As illustrated in Figure~\ref{fig:actual_vs_predicted}, the model demonstrates high predictive fidelity across three orders of magnitude of energy consumption ($10^{-3}$ to $10^{-1}$ kWh). The observations are tightly clustered around the 1:1 perfect prediction line, confirming that the log-linear formulation effectively captures the underlying physical energy dynamics.

The final estimated energy model for training is
\begin{equation}
\label{trans_eq_final}
E
\;\approx\;
e^{-10.52}
\;\cdot\;
C^{0.20}
\;\cdot\;
M^{0.10}
\;\cdot\;
\eta_h^{-1.06}
\;\cdot\;
\Psi
\end{equation}
where $C$ denotes total compute (FLOPs), $M$ is the memory-traffic proxy, $\eta_h$ is the hardware-efficiency factor derived from speedup, and
$\Psi$ is a strategy-dependent constant.

\begin{figure}[ht]
\centering
\includegraphics[width=\columnwidth]{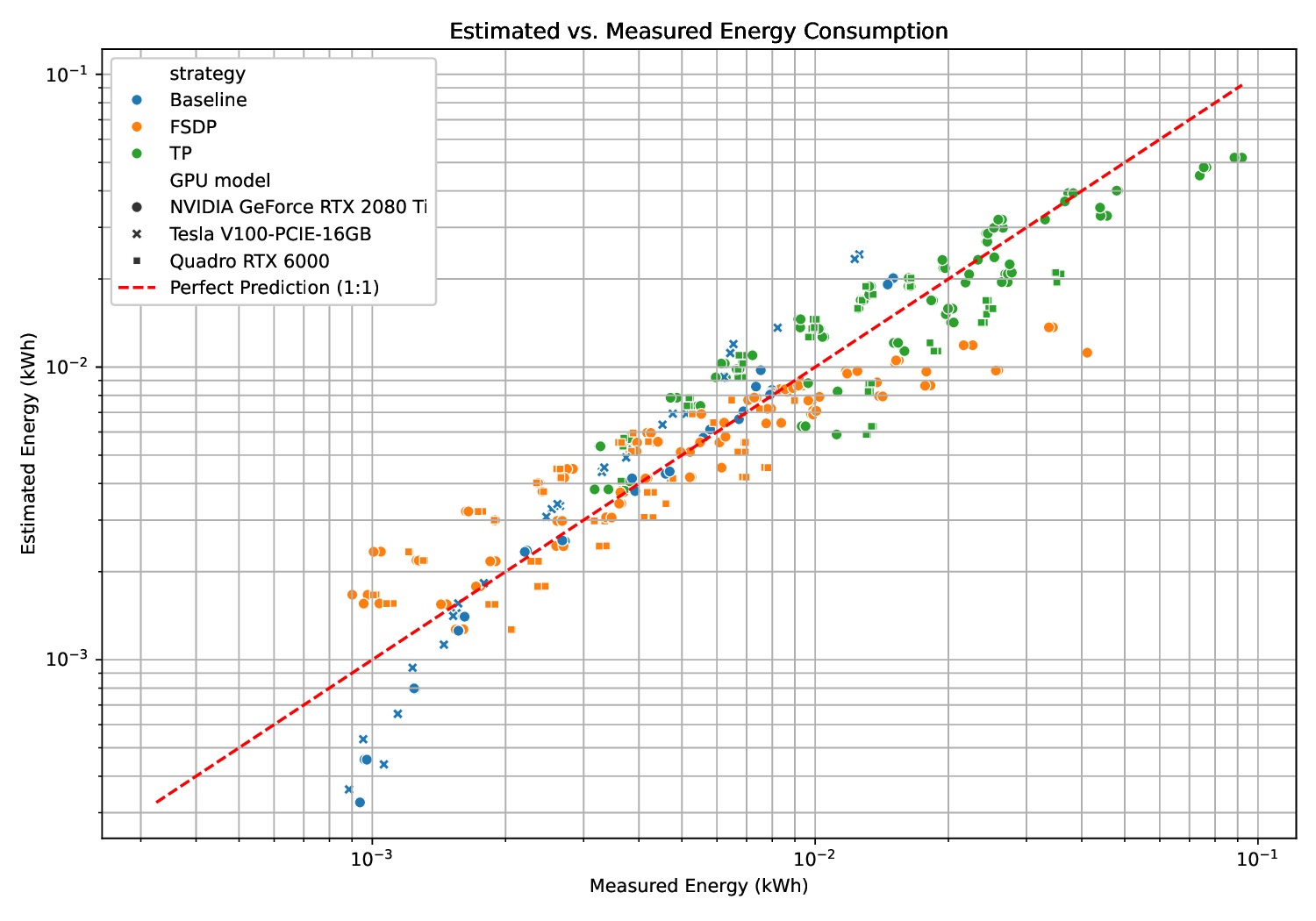}    
\vspace{-20pt}
\caption{Estimated vs. Measured energy consumption in kWh. The dashed red line represents the 1:1 ideal prediction. The model maintains high accuracy across diverse GPU architectures and distributed strategies, successfully capturing the scaling behavior of the system ($R^2_{\mathrm{adj}} = 0.853$).}
\label{fig:actual_vs_predicted}
\vspace{-15pt}
\end{figure}

\subsection{Interpretation of Energy Scaling Results}
\label{subsec:results-interpretation}

We interpret the estimated coefficients of the energy model in terms of their implications for increases or decreases in training energy. Because the model is specified in log space, coefficients correspond to elasticities or multiplicative
effects and can be interpreted quantitatively, conditional on a fixed workload and hardware class.

\textbf{Effect of compute.}
The estimated elasticity of energy with respect to total compute is
$\hat{\alpha}_C = 0.20$, indicating strongly sub-linear scaling. A 10\% increase in compute leads to an approximate 2\% increase in energy, while doubling compute increases energy by only about 15\%. This behavior suggests that larger workloads benefit from improved utilization and amortization of fixed overheads, consistent
with roofline-style efficiency effects.

\textbf{Effect of memory traffic.}
The coefficient on the memory-traffic proxy is positive but modest
($\hat{\alpha}_M = 0.10$). Doubling memory traffic increases energy consumption by
approximately 7\%, indicating that memory intensity contributes to energy use, but to a lesser extent than compute or efficiency effects in the regimes studied. This is consistent with partially memory-bound execution on modern GPUs.

\textbf{Effect of hardware efficiency.}
Hardware efficiency has the strongest impact on energy consumption. The coefficient is $\hat{\alpha}_{\eta} = -1.06$, implying that improvements in
parallel efficiency translate almost directly into energy savings. A 10\%
increase in efficiency reduces energy by approximately 10\%, while a twofold increase reduces energy by about 52\%. This result empirically confirms the central role of execution-time scaling in determining training energy.

\textbf{Effect of parallelization strategy.}
Parallelization strategy introduces substantial fixed effects beyond those captured
by compute, memory traffic, and efficiency. Relative to the single-GPU baseline,
both fully sharded data parallelism (FSDP) and tensor parallelism (TP) are associated
with large positive multiplicative shifts
($\Psi=1$ for the baseline, $\Psi\approx e^{5.38}$ for FSDP, and $\Psi\approx e^{7.65}$
for TP), reflecting additional communication and coordination overheads inherent to
distributed execution. Parallelism therefore reduces energy only when efficiency
gains are sufficient to offset these strategy-specific costs.

Overall, the results show that training energy is far more sensitive to hardware efficiency than to raw increases in compute or memory traffic. While larger models and workloads do increase energy consumption, their impact is sub-linear and can be mitigated by effective parallelization. In contrast, poorly utilized parallel execution can substantially increase energy despite reduced time-to-solution, underscoring the importance of jointly modeling compute, memory, and efficiency when evaluating the energy cost of training.

%% file: validation.tex
\section{Validation, Diagnostics, and Ablations}
\label{sec:validation}
\paragraph{No-$\eta_h$ baseline.}
To quantify the value of modeling hardware efficiency, we compare a baseline power-law model,
$E=\kappa,C^{\alpha_C}M^{\alpha_M}$, against the proposed full model that augments compute and memory with a speedup-derived hardware-efficiency term and architecture effects. Including $\eta_h$ yields a clear improvement in predictive
accuracy (Appendix Table~\ref{tab:model_comparison}): $R^2$ increases from
$0.7768$ to $0.8543$ and RMSE decreases by 19\%. Calibration remains unchanged (unit slope), indicating reduced variance rather than rescaling.
A likelihood-ratio test strongly favors the full model, showing that a substantial
portion of energy variation arises from implementation-specific throughput and
parallelization effects beyond compute and memory scaling alone.

\vspace{-0.4cm}

\paragraph{Estimation on state-of-the-art models.}
\label{subsec:public-validation}

We estimate training energy for several prominent language models (T5, Meena, GShard, Switch Transformer, and GPT-3) under a fixed-data regime that isolates architectural and parallelization effects. Because the original models were
trained on TPU clusters, these results should be interpreted as order-of-magnitude consistency checks rather than direct validation. Still, the estimates preserve reported scaling trends and relative efficiency differences across strategies, suggesting that the framework captures dominant training-energy drivers in this controlled setting. Results are shown in Table~\ref{tab:energy_comparison}, with
additional details in Appendix~\ref{app:sota-energy}.

\begin{table}[ht]
\centering
\caption{Comparison of reported and estimated energy consumption for SOTA models.}
\label{tab:energy_comparison}
\scriptsize 
\setlength{\tabcolsep}{3pt} 
\resizebox{\columnwidth}{!}{%
\begin{tabular}{llrrrccrrrrr}
\toprule
\textbf{Model} & \textbf{Dev.} & \textbf{Params} & \textbf{FLOPs} & \textbf{Tokens} & \textbf{C\_epoch} & \textbf{M\_proxy} & \textbf{Rep. Energy} & \multicolumn{2}{c}{\textbf{FSDP Est. (MWh)}} & \multicolumn{2}{c}{\textbf{TP Est. (MWh)}} \\
 & & (B) & (ZF) & /Epoch & (FLOPs) & (Proxy) & (MWh) & Epoch. & Tot. & Epoch. & Tot. \\
\midrule
GPT-3  \cite{brown2020language}& OpenAI & 175.0 & 314.0& 1e14 & 1.81e25 & 2.42e23 & 1287.0 & 0.12 & 17.33 & 4.88 & 696.13 \\
GShard \cite{lepikhin2020gshard} & Google & 619.0 & 133.3 & 1e14 & 2.57e23 & 5.03e21 & 24.1 & 0.04 & 1.72 & 1.45 & 69.08 \\
Meena \cite{adiwardana2020towards}& Google & 2.6 & 112.0 & 1e14 & 3.21e22 & 6.29e20 & 232.0 & 0.02 & 6.90 & 0.79 & 276.94 \\
Switch \cite{fedus2022switch}& Google & 1500.0 & 82.2  & 1e14 & 1.71e23 & 3.36e21 & 179.0 & 0.03 & 8.40 & 1.29 & 337.34 \\
T5  \cite{raffel2020exploring}& Google & 11.0 & 40.5 & 1e14 & 3.40e22 & 1.26e21 & 85.7 & 0.02 & 2.65 & 0.85 & 106.47 \\
\bottomrule
\end{tabular}
}
\vspace{-0.4cm}

\end{table}


%% file: Discussion.tex
\section{Discussion}
\label{sec:discussion}

By decomposing training energy into contributions from compute, memory traffic, and
execution efficiency, we explain why energy scaling departs from FLOPs-only
intuition and identify the system-level factors that most strongly influence
training efficiency. Our results show that energy inefficiency arises not only from
compute intensity, but also from memory pressure and parallelization overheads,
motivating system-aware energy models rather than purely model-centric ones.
\vspace{-0.35cm}

\paragraph{Beyond FLOPs and runtime-based scaling.}
Across all experiments, hardware efficiency is the dominant driver of training energy variation. While total compute remains a necessary descriptor of workload
size, its effect on energy is strongly sub-linear once utilization is accounted for. Faster configurations often operate at higher instantaneous power, offsetting runtime reductions and showing that neither FLOPs nor runtime alone is a sufficient proxy for energy cost.
\vspace{-0.35cm}

\paragraph{Parallelism and diminishing returns.}
Parallelism introduces a fundamental trade-off between time-to-solution and energy
efficiency. Although tensor parallelism and fully sharded data parallelism reduce
runtime, they also incur coordination, communication, and synchronization overheads
that can increase total energy unless parallel efficiency remains high. As a
result, scaling to larger device counts can be counterproductive from an energy
perspective if the chosen strategy is poorly matched to model structure, hardware,
or interconnect regime.

\textbf{Architectural scaling and strategy choice.}
Fully sharded data parallelism consistently achieves lower training energy than
tensor parallelism, particularly for larger and deeper models. This suggests that
sharding parameters and optimizer state is often more energy efficient than
partitioning computation across devices. These differences amplify with model
scale, indicating that architectural scaling and parallelization strategy must be
co-designed for energy-efficient training.

\textbf{Limitations.}
Our empirical evaluation is limited to BERT-style Transformer fine-tuning under
controlled architectural sweeps. While the modeling structure is not tied to BERT
specifically, we do not validate it on decoder-only pretraining,
mixture-of-experts models, long-context attention variants, or pipeline-parallel
execution; results on larger public models are therefore presented as trend-level
consistency checks rather than direct validation.

The model also relies on coarse-grained proxies for compute and memory demand,
chosen as minimal sufficient statistics rather than kernel-level attributions.
Accordingly, fitted coefficients and strategy-specific constants should be
interpreted as aggregate effects reflecting communication, synchronization, and
execution overheads. Finally, the speedup energy linkage assumes operation within
a fixed hardware class and operating regime, where average device power varies more
slowly than execution time; it does not account for embodied emissions,
data-center level effects, or heterogeneous power-management behavior.

\textbf{Broader implications and outlook.}
Our findings suggest that reducing the environmental footprint of large-scale
training requires a shift from model-centric metrics toward system-aware
efficiency considerations. Reporting execution time and parallel efficiency
alongside FLOPs and parameter counts would enable more meaningful comparisons
across models and hardware platforms. Future work could extend this framework to
multi-node and heterogeneous settings, incorporate finer-grained communication and
attention-kernel models, and integrate carbon-intensity aware scheduling for more
comprehensive energy-aware system design.

%% file: conclusion.tex

\section{Conclusion}
\label{sec:conclusion}

This work shows that the energy cost of Transformer training is governed not only by
total compute, but also by hardware efficiency and parallelization strategy. While
distributed training reduces time-to-solution, it can increase total energy when parallel scaling is imperfect or coordination overheads dominate. By modeling
training energy as a function of compute, memory traffic, and a speedup-derived
hardware-efficiency term, we move beyond FLOPs-only scaling laws and provide an
interpretable framework for energy prediction.
Our results further show that execution-time scaling can serve as a useful proxy
for energy efficiency within a fixed hardware class and operating regime, enabling
out-of-sample energy estimation even when direct power measurements are unavailable.
Overall, the framework suggests that improving parallel efficiency and utilization
can be as important as reducing nominal compute, offering practical guidance for
the energy-aware design and evaluation of distributed Transformer training systems.

\section*{Impact Statement}
This paper studies the operational energy consumption of Transformer training and proposes an interpretable framework for predicting training energy from compute, memory traffic, and execution efficiency. A potential positive impact is to support more energy-aware design and reporting of machine learning systems, including better comparison of training configurations beyond FLOPs or runtime alone.

At the same time, this work focuses only on operational training energy and does not account for embodied emissions, electricity mix, or downstream inference costs. As a result, it should not be used in isolation as a complete measure of environmental impact. More generally, improved energy prediction could also make large-scale training easier to optimize, which may reduce energy per run without necessarily reducing total energy use. We therefore view this work as a tool for improving transparency in system-level trade-offs rather than as a complete
sustainability assessment.

%% file: results_appendix.tex
\section{ Pre-Modeling Exploratory Data Analysis}

\subsection{Energy Consumption vs.\ Model Size}
\label{subsec:energy-params}
Figure~\ref{fig:energyVsParams} illustrates the relationship between training energy
consumption and model size for Transformer models under different parallelization
strategies and GPU counts. Across all configurations, energy consumption increases
monotonically with the number of parameters, confirming that model size is a
first-order driver of training energy. However, the slope and dispersion of this
increase depend strongly on the parallelization strategy.

For a fixed number of GPUs, FSDP consistently exhibits lower energy consumption than
TP at comparable model sizes. This gap widens as the model grows, indicating that
tensor parallelism incurs increasing overheads such as additional collective
communications and reduced kernel efficiency  that are not fully amortized by larger
models. In contrast, FSDP scales more smoothly with parameter count, suggesting
better preservation of hardware utilization as model size increases.

Increasing the number of GPUs from $2$ to $4$ raises the absolute energy consumption
for both strategies, despite reducing time-to-solution. This highlights a key
trade-off: while additional devices improve throughput, they also increase
instantaneous power draw and communication costs, leading to higher total energy
usage when parallel efficiency is imperfect.

Notably, the energy–parameter relationship is sub-linear on a log–log scale,
consistent with the scaling exponent $\alpha_C<1$ estimated in
Section~\ref{subsec:estimation}. This behavior reflects the combined effects of
roofline-style utilization improvements and the amortization of fixed overheads at
larger model sizes. Overall, the figure underscores that model size alone is
insufficient to predict energy consumption; parallelization strategy and hardware
efficiency play a decisive role in determining the energy cost of training.

\subsection{Energy Consumption vs.\ Model Compute}
\label{subsec:energy-flops}
Figure~\ref{fig:energyVsFlops} shows how total training energy scales with model
compute across parallelization strategies and GPU counts. As expected, energy
consumption increases monotonically with the total number of floating-point
operations, confirming compute as a primary driver of training energy. However,
the relationship is markedly sub-linear and exhibits substantial dispersion,
indicating that FLOPs alone are insufficient to predict energy usage.

For a fixed compute budget, FSDP consistently consumes less energy than TP across
both $2$-GPU and $4$-GPU configurations. This gap becomes more pronounced at higher
compute levels, suggesting that tensor parallelism incurs increasing coordination
and communication overheads that are not fully amortized as compute grows. In
contrast, FSDP better preserves hardware utilization by reducing redundant
parameter storage and communication, leading to improved energy efficiency.

Increasing the number of GPUs from $2$ to $4$ shifts the curves upward for both
strategies, despite reducing wall-clock training time. This highlights a key
system-level trade-off: additional devices improve throughput but increase
instantaneous power draw and synchronization overhead, resulting in higher total
energy when parallel efficiency is imperfect.

Finally, the curvature observed on a log--log scale aligns with the sub-linear
compute elasticity estimated in Section~\ref{subsec:estimation}. This behavior is consistent with roofline-style utilization effects and reinforces the necessity of
explicitly modeling hardware efficiency  via $\eta_h$  to explain energy scaling
beyond raw compute counts.

\begin{figure}[ht]
    \centering
    \includegraphics[width=\columnwidth]{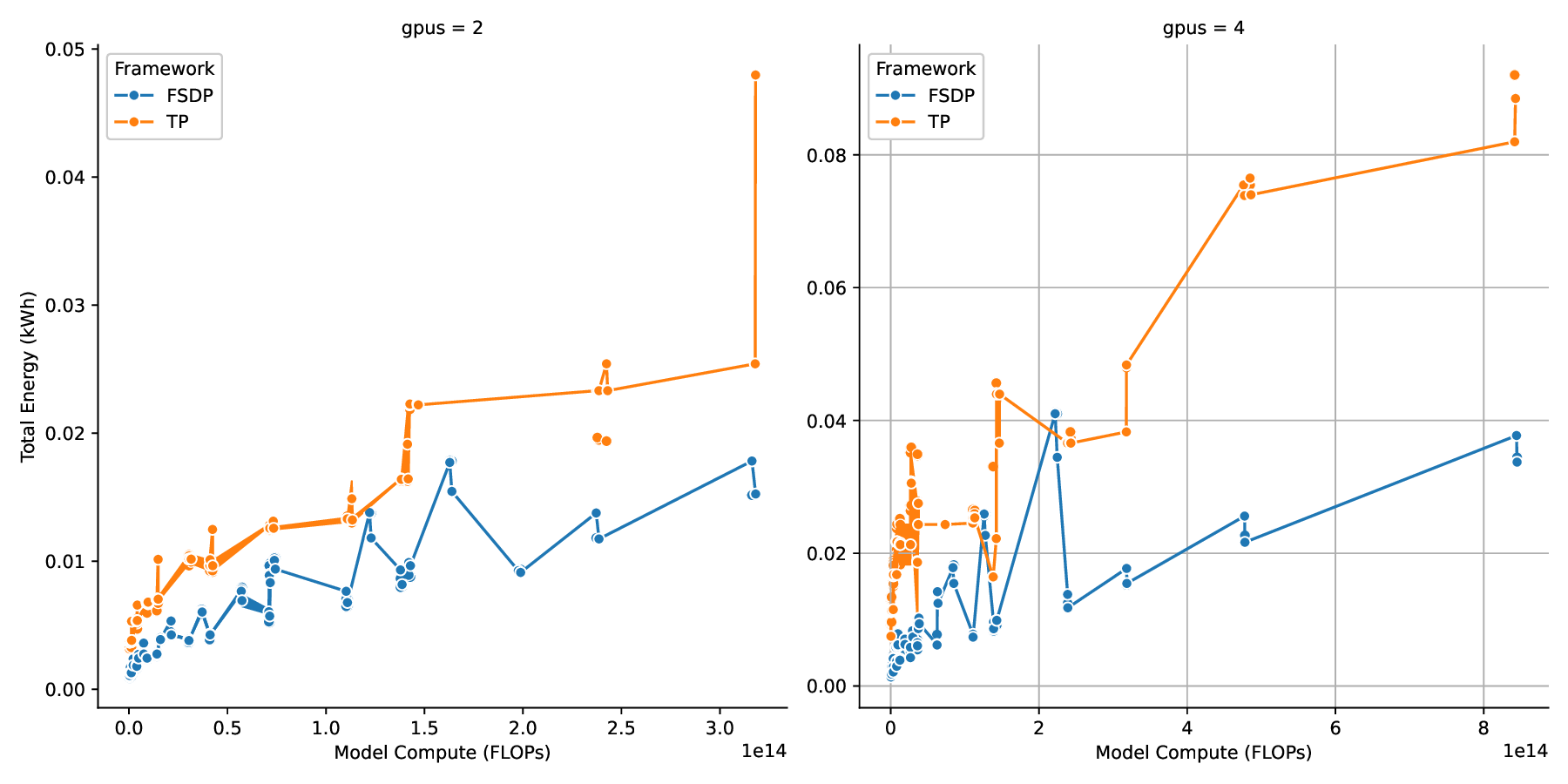}
    \vspace{-20pt}
    \caption{
    Training energy consumption as a function of total model compute (FLOPs) for
    fully sharded data parallelism (FSDP) and tensor parallelism (TP), shown
    separately for $2$ and $4$ GPUs. 
    }
    \label{fig:energyVsFlops}
\end{figure}

\subsection{Energy Consumption vs.\ Training Duration}
\label{subsec:energy-duration}
Figure~\ref{fig:energyVsDuration} compares training duration across baseline,
fully sharded data parallelism (FSDP), and tensor parallelism (TP) as the number of
GPUs increases. As expected, increasing the number of GPUs generally reduces
training duration for both parallel strategies, reflecting improved throughput
and higher aggregate compute capacity.

However, the reduction in duration does not translate proportionally into lower
energy consumption. While FSDP achieves the shortest median durations for $2$ and
$4$ GPUs, TP exhibits substantially longer and more variable runtimes, particularly
at higher GPU counts. This variability indicates sensitivity to communication
patterns and synchronization overheads introduced by tensor parallelism.

Crucially, these results highlight that minimizing time-to-solution is not
equivalent to minimizing energy. Faster configurations often operate at higher
instantaneous power due to increased device utilization and communication
activity, which can offset  or even outweigh  the benefits of reduced runtime. This
decoupling between duration and energy motivates the explicit inclusion of a
hardware-efficiency factor in our energy model, as duration alone cannot capture
the full cost of parallel execution.

Overall, Figure~\ref{fig:energyVsDuration} reinforces the need to jointly consider
runtime, power draw, and parallel efficiency when evaluating the sustainability of
training configurations.
\textbf{Energy vs.\ training duration.}
Figure~\ref{fig:energyVsDuration} highlights a clear decoupling between runtime and
energy. Faster configurations often consume more energy, demonstrating that
minimizing time-to-solution does not necessarily minimize energy usage. This
motivates explicitly modeling hardware efficiency rather than relying on runtime
as a proxy.
\begin{figure}[ht]
    \centering
    \includegraphics[width=\columnwidth]{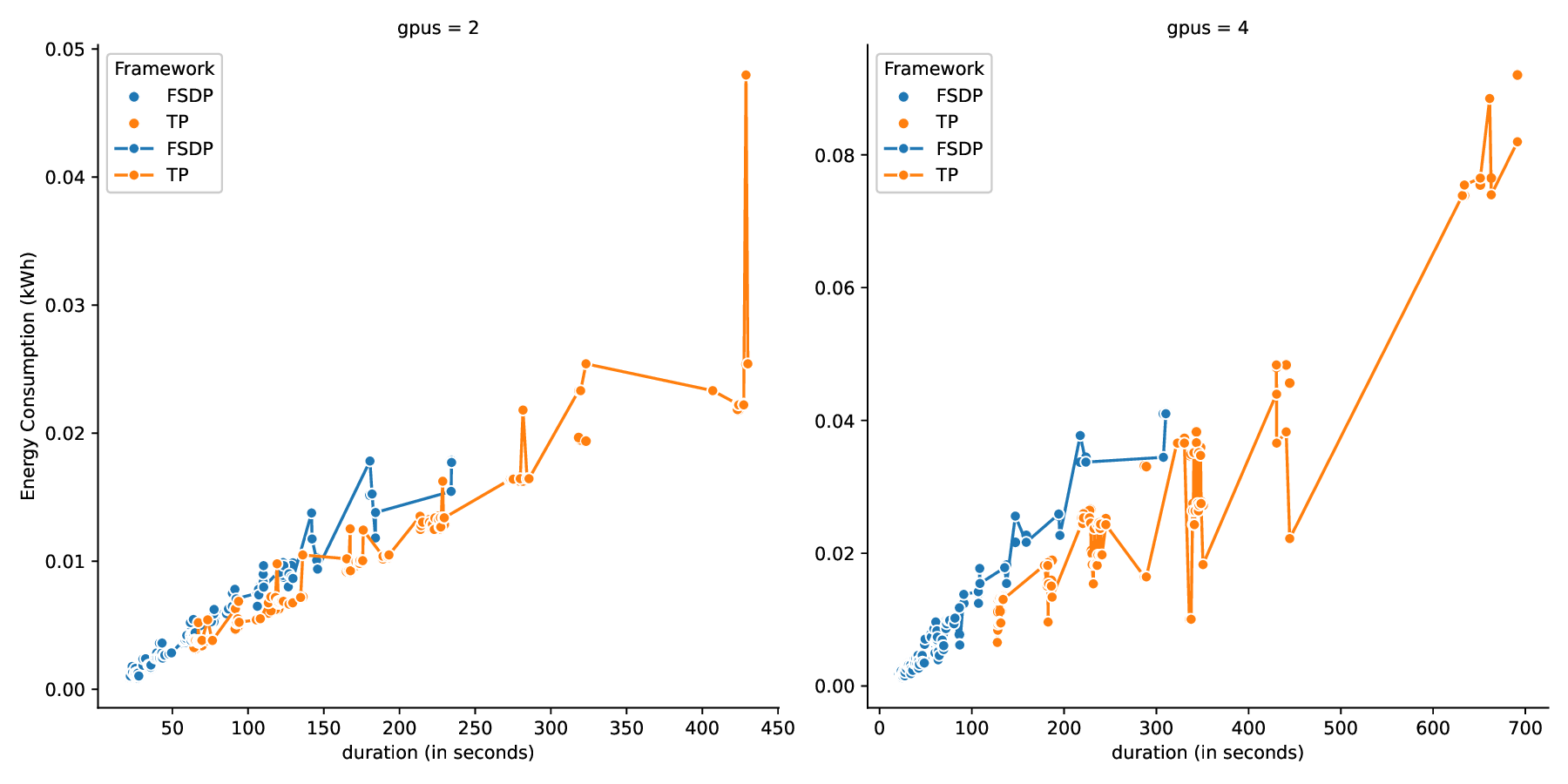}
    \caption{
    Distribution of training duration (seconds per epoch) across parallelization
    strategies and GPU counts. B
    }
    \label{fig:energyVsDuration}
    \vspace{-10pt}
\end{figure}

\subsection{Energy Consumption vs.\ Computational Resources}
\label{subsec:energy-gpus}
Figure~\ref{fig:energyVsGpus} illustrates how total energy consumption varies with
the number of GPUs for different parallelization strategies. Moving from a single
GPU to multiple GPUs increases total energy consumption across all configurations,
even though multi-GPU training substantially reduces wall-clock time. This confirms
that additional computational resources incur higher instantaneous power draw and
coordination overheads that are not fully offset by speedup gains.

Among multi-GPU strategies, FSDP consistently exhibits lower median energy
consumption than tensor parallelism at the same GPU count. This difference becomes
more pronounced at four GPUs, where TP shows both higher medians and substantially
larger variance. The increased dispersion under TP reflects sensitivity to
communication patterns and reduced kernel efficiency as tensors are partitioned
across devices.

Notably, the single-GPU baseline remains the most energy-efficient configuration in
absolute terms for small- and medium-scale models, despite being the slowest in
runtime. This highlights a central trade-off in distributed training: optimizing
for time-to-solution does not necessarily minimize energy consumption.

Overall, these results emphasize that the energy cost of training grows with the
number of devices unless parallel efficiency is sufficiently high. This motivates
the explicit inclusion of a device-count–dependent hardware-efficiency factor
$\eta_h$ in our energy model to capture diminishing returns from additional GPUs.

\subsection{Energy Consumption vs.\ Number of Layers}
\label{subsec:energy-layers}
Figure~\ref{fig:energyVslayersGpus} examines how training energy scales with model
depth under different parallelization strategies and GPU counts. For all
configurations, energy consumption increases monotonically with the number of
layers, confirming depth as a direct driver of training cost through its linear
effect on both parameter count and total compute.

At fixed depth, increasing the number of GPUs leads to higher total energy
consumption across all strategies. Although multi-GPU execution reduces
wall-clock time, the increase in instantaneous power draw and synchronization
overheads outweighs these gains, resulting in higher overall energy usage. This
effect becomes more pronounced as depth increases, reflecting the compounding
impact of deeper models on communication and memory traffic.

Across all depths and GPU counts, FSDP consistently achieves lower energy
consumption than tensor parallelism. The gap between FSDP and TP widens for deeper
models, indicating that TP incurs depth-dependent overheads  such as per-layer
collective operations and reduced kernel efficiency  that scale poorly with model
depth. In contrast, FSDP better amortizes communication costs by sharding
parameters and optimizer states, leading to more stable energy scaling.

Notably, for shallow models (4--6 layers), the single-GPU baseline remains the most
energy-efficient configuration in absolute terms. As model depth increases,
however, the baseline becomes increasingly impractical due to long runtimes,
making multi-GPU strategies necessary despite their higher energy cost. This
highlights a fundamental trade-off between feasibility and energy efficiency in
training deep models.

Overall, Figure~\ref{fig:energyVslayersGpus} reinforces that model depth,
parallelization strategy, and device count interact non-trivially to determine
energy consumption. These results motivate modeling depth-driven compute separately
from hardware efficiency effects, as captured by the $\eta_h$ factor in our energy
model.
\begin{figure}[ht]
    \centering
    \includegraphics[width=\columnwidth]{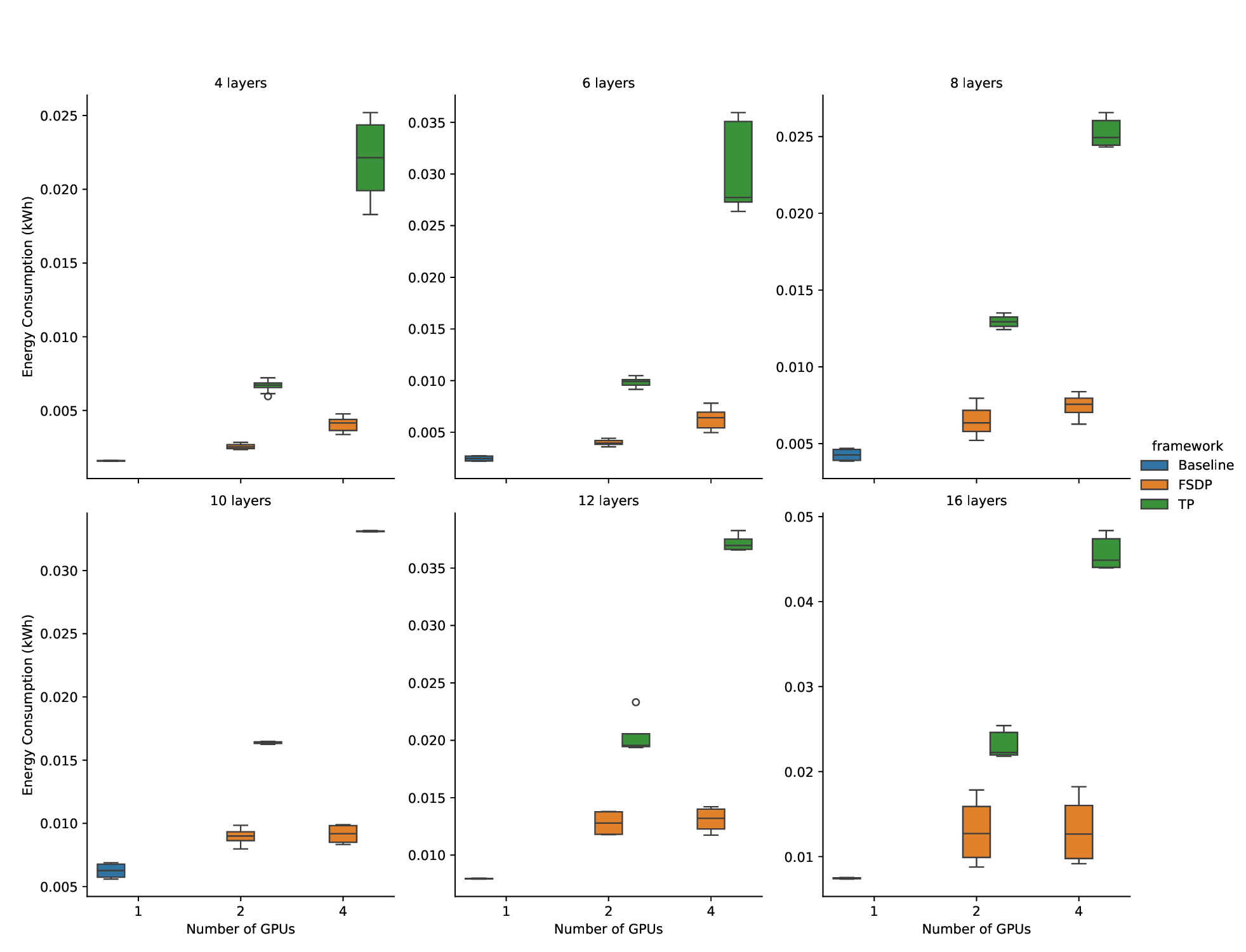}
    \caption{
    Distribution of training energy consumption as a function of model depth
    (number of Transformer layers) and GPU count. Each panel corresponds to a
    fixed depth; box plots show medians and interquartile ranges across baseline,
    fully sharded data parallelism (FSDP), and tensor parallelism (TP).
    }
    \label{fig:energyVslayersGpus}
\end{figure}

\section{Estimated Compute--Memory Energy Model}
\label{app:comp-mem}

This appendix reports the estimated parameters of the linear compute--memory energy decomposition introduced in Eq.~\eqref{eq:energy-decomp-main}. The model
expresses per-step training energy as an additive combination of compute cost,memory-traffic cost, and fixed overheads, and serves as a mechanistic baseline for
the reduced-form scaling laws used in the main text.

\begin{table}[ht]
\centering
\caption{Estimated coefficients of the compute--memory energy model
(Eq.~\eqref{eq:energy-decomp-main}) using HC3 robust standard errors.}
\label{tab:ols_results}
\small
\begin{tabular}{lccccc}
\toprule
\textbf{Variable} & \textbf{Coef.} & \textbf{Std. Err.} & \textbf{$z$} & \textbf{$P>|z|$} & \textbf{95\% CI} \\
\midrule
$E_{0}$ (Intercept)                   & 0.0028 & 0.0010 & 3.72  & 0.000 & [0.001, 0.004] \\
Strategy: FSDP                        & 0.0041 & 0.0010 & 4.59  & 0.000 & [0.002, 0.006] \\
Strategy: TP                          & 0.0163 & 0.0012 & 14.07 & 0.000 & [0.014, 0.019] \\
Compute ($C$, centered)               & 0.0027 & 0.0010 & 4.38  & 0.000 & [0.001, 0.004] \\
Memory ($M_{\text{proxy}}$, centered) & $3.0 \times 10^{-4}$ & $4.4 \times 10^{-5}$ & 5.96 & 0.000 & [$2.1 \times 10^{-4}$, $3.9 \times 10^{-4}$] \\
\bottomrule
\end{tabular}
\end{table}

The results confirm that both compute and memory traffic contribute positively and
significantly to training energy, consistent with the additive structure of the
model. The memory-traffic proxy exhibits a statistically significant coefficient,
indicating that off-chip data movement contributes non-negligibly to energy
consumption even after controlling for total compute.

Parallelization strategy introduces substantial additive offsets. Relative to the
single-GPU baseline, both fully sharded data parallelism (FSDP) and tensor
parallelism (TP) are associated with higher per-step energy, reflecting additional
communication, synchronization, and coordination overheads that are not captured by
raw compute or memory proxies alone. The substantially larger coefficient for TP is
consistent with its heavier reliance on fine-grained collective communication.

While the compute--memory model explains a majority of the variance in per-step
energy, its adjusted $R^2$ of $0.694$ is significantly lower than that of the
reduced-form model incorporating hardware efficiency. This gap motivates the
introduction of the speedup-based hardware-efficiency factor $\eta_h$ in the main
text, which captures execution-time effects and parallel efficiency beyond static
compute and memory costs.

\section{Final Estimated Energy Model}
\label{app:final-model}

Unless otherwise stated, all reported results use the reduced-form model
Eq.~\eqref{eq:reduced-form-main}, with the hardware-efficiency term $\eta_h$
parameterized via empirical speedup models (Eq.~\eqref{eq:speedup}). Estimated
coefficients and confidence intervals are summarized in
Table~\ref{tab:exponents}. The model achieves an adjusted coefficient of
determination $R^2_{\mathrm{adj}} = 0.853$ over 410 independent configurations,
indicating strong explanatory power across heterogeneous architectures and
parallelization strategies.
As illustrated in Figure~\ref{fig:actual_vs_predicted}, the model demonstrates high predictive fidelity across three orders of magnitude of energy consumption ($10^{-3}$ to $10^{-1}$ kWh). The observations are tightly clustered around the 1:1 perfect prediction line, confirming that the log-linear formulation effectively captures the underlying physical energy dynamics.
The estimated coefficients reveal a strongly sub-linear dependence of energy on total compute, with elasticity $\hat{\alpha}_C = 0.196$This reduced elasticity reflects the fact that a substantial portion of energy scaling with workload size is mediated through changes in hardware utilization, which are explicitly captured by the parallel-efficiency term. Once utilization is accounted for, the residual sensitivity of energy to raw FLOP count is markedly attenuated.

Memory traffic also contributes positively to energy consumption. The coefficient on $\log(M_{\mathrm{proxy}})$ is positive and statistically significant, indicating that configurations with larger activation footprints incur higher energy costs, consistent with memory-bound regimes predicted by roofline analysis.

Hardware efficiency plays a dominant role. The coefficient on
$\log(\text{parallel efficiency})$ is
$\hat{\alpha}_{\eta} = -1.06 $,
implying that improvements in parallel efficiency translate directly into reduced
energy consumption. This result empirically validates
Lemma~\ref{lem:speedup_to_energy} and confirms that execution-time scaling is a
first-order determinant of training energy.

Parallelization strategy effects are substantial. Relative to the single-GPU
baseline, both fully sharded data parallelism (FSDP) and tensor parallelism (TP)
exhibit significant upward shifts in energy consumption when controlling for
compute, memory, and efficiency. Specifically, the estimated multiplicative
factors are
$
\Psi_{\mathrm{FSDP}} \approx \exp(5.38)$,$
\Psi_{\mathrm{TP}} \approx \exp(7.65),
$
reflecting the additional communication and coordination overheads introduced by
distributed execution beyond what is captured by speedup alone. These effects
highlight the importance of explicitly modeling strategy-dependent constants in
energy prediction.

\begin{table}[ht]
\centering
\caption{
Estimated coefficients for the Transformer energy model.
The last two columns report 95\% confidence intervals.
The model attains $R^2_{\mathrm{adj}} = 0.853$ over 410 configurations.
}
\label{tab:exponents}

\resizebox{\columnwidth}{!}{%
\begin{tabular}{l|cccccc}
\hline
\textbf{Variable} & \textbf{Coef.} & \textbf{Std. Err.} & \textbf{z} & \textbf{P$>|z|$} & \textbf{[0.025} & \textbf{0.975]} \\ \hline
Intercept & -10.5283 &  0.449 & -23.472 & 0.000 & -11.407& -9.649 \\
Strategy: FSDP & 5.3839 & 0.466 & 11.556 & 0.000 & 4.471 & 6.297 \\
Strategy: TP & 7.6487 & 0.581 & 13.165 & 0.000 & 6.510 & 8.787 \\
$\log(\text{compute})$ & 0.1955 & 0.025 & 7.881 & 0.000 & 0.147 & 0.244 \\
$\log(M_{\mathrm{proxy}})$ & 0.0980 & 0.037 & 2.637 & 0.008 & 0.025 & 0.171 \\
$\log(\text{parallel efficiency})$ & -1.0601 & 0.099 & -10.678 & 0.000 & -1.255 & -0.865 \\ \hline
\multicolumn{5}{l}{\small \textit{Notes: Standard errors are HC3 robust. Continuous variables are mean-centered.}} \\
\multicolumn{5}{l}{\small \textit{Adjusted $R^2 = 0.8525$; Durbin-Watson = 2.125; Condition Number = 86.5.}} \\
\end{tabular}
}
\end{table}

\section{Model Diagnostics and Robustness}
\label{app:diag}
\paragraph{No-$\eta_h$ baseline.}
To isolate the contribution of hardware-efficiency modeling, we compare a
baseline power-law model,
$E=\kappa\,C^{\alpha_C} M^{\alpha_M}$, to the proposed full specification that
augments compute and memory with a speedup-derived hardware-efficiency term and
architecture-specific effects. As shown in Appendix Table~\ref{tab:model_comparison}, including $\eta_h$ substantially improves
predictive performance: explained variance increases from
$R^2=0.7768$ to $R^2=0.8543$ ($+\;0.078$), while RMSE decreases by
$0.091 kwh$ (a 19\% reduction). Calibration remains unchanged, with a unit
slope in both models, indicating that the gains arise from reduced dispersion
rather than rescaling. A likelihood-ratio test strongly favors the full model, demonstrating that a substantial fraction
of energy variation is driven by implementation-specific throughput and
parallelization effects rather than compute and memory scaling alone.

\begin{table}[ht]
\centering
\caption{Model comparison: baseline power law vs.\ full energy model}
\label{tab:model_comparison}
\resizebox{\columnwidth}{!}{%
\begin{tabular}{lccc}
\hline
\textbf{Metric} & \textbf{Power Law (A)} & \textbf{Full Model (B)} & \textbf{$\Delta$ (B--A)} \\ \hline
$R^2$ & 0.7768 & 0.8543 & +0.0775 \\
RMSE (kwh) & 0.4761 & 0.3846 & $-0.0914$ \\
Calibration slope & 1.0000 & 1.0000 & 0.0000 \\ \hline
\multicolumn{4}{l}{\textit{Likelihood-ratio test statistic: 174.90}} \\
\multicolumn{4}{l}{\textit{LRT $p$-value: $6.30\times10^{-40}$}} \\ \hline
\end{tabular}
}
\end{table}
To validate the reliability of our energy scaling laws, we conducted a multi-dimensional diagnostic analysis of the Ordinary Least Squares (OLS) residuals. The results are reported in figures \ref{fig:residuals}.

\paragraph{Linearity and Functional Form:} The \textit{Residuals vs. Fitted} plot reveals a slight U-shaped curvature in the LOWESS trend line. While the log-log specification captures the primary scaling trend (Adj. $R^2 = 0.853$), these marginal deviations at the extreme ends of the energy spectrum suggest that the scaling exponent ($\alpha$) may undergo subtle shifts as models transition between small-scale and large-scale regimes.

\paragraph{Residual diagnostics}
To evaluate the reliability of the energy scaling law, a comprehensive diagnostic analysis was performed on the model's residuals, as illustrated in Figure \ref{fig:residuals} in appendix. The model demonstrates a robust log-log fit with an Adjusted  of 0.853 and a Durbin-Watson statistic of 2.125, confirming the absence of first-order autocorrelation. While the \textit{Residuals vs. Fitted} plot reveals minor non-linearity at the extreme ends of the compute scale and the \textit{Normal Q-Q} plot indicates heavy-tailed residuals typical of hardware power spikes, the error distribution remains stable across different training strategies and architectural proxies.

\paragraph{Feature-Specific Bias:} The \textit{Residuals vs. $\log$ Compute} and \textit{Residuals vs. $\log M_{proxy}$ (Centered)} plots show relatively balanced dispersion across the centered compute range. However, a minor upward trend in residuals is observed for the highest model-size proxies, indicating that for exceptionally large architectural dimensions, the model may slightly underestimate energy consumption.

\paragraph{Strategy-Level Stability:} The \textit{Residuals by Strategy} boxplot demonstrates that prediction error is consistent across \textit{Baseline}, \textit{FSDP}, and \textit{TP}. While distributed strategies exhibit slightly higher variance and a few isolated outliers, the median residual for each category remains centered at zero, justifying our use of strategy-specific intercept shifts to account for parallelism overhead.

\paragraph{Error Structure and Normality:} The \textit{Normal Q-Q Plot} displays heavy-tailed residuals (leptokurtosis), where extreme energy outcomes deviate from the Gaussian line at the tails. This is consistent with hardware benchmarking data subjected to transient power spikes. Furthermore, a Durbin-Watson statistic of 2.125 confirms that our aggregation strategy successfully eliminated first-order autocorrelation.

\paragraph{Influence and Numerical Stability:} The \textit{Cook's Distance} plot identifies two specific clusters of influential observations (near indices 60 and 265). Despite these clusters, all distances remain below 0.08, well within safe thresholds for regression stability. Finally, mean-centering continuous variables reduced the Condition Number to 86.5, ensuring that the effects of compute and parallel efficiency are numerically decoupled and individually interpretable.

\begin{figure}[t]
\centering
\includegraphics[width=\columnwidth]{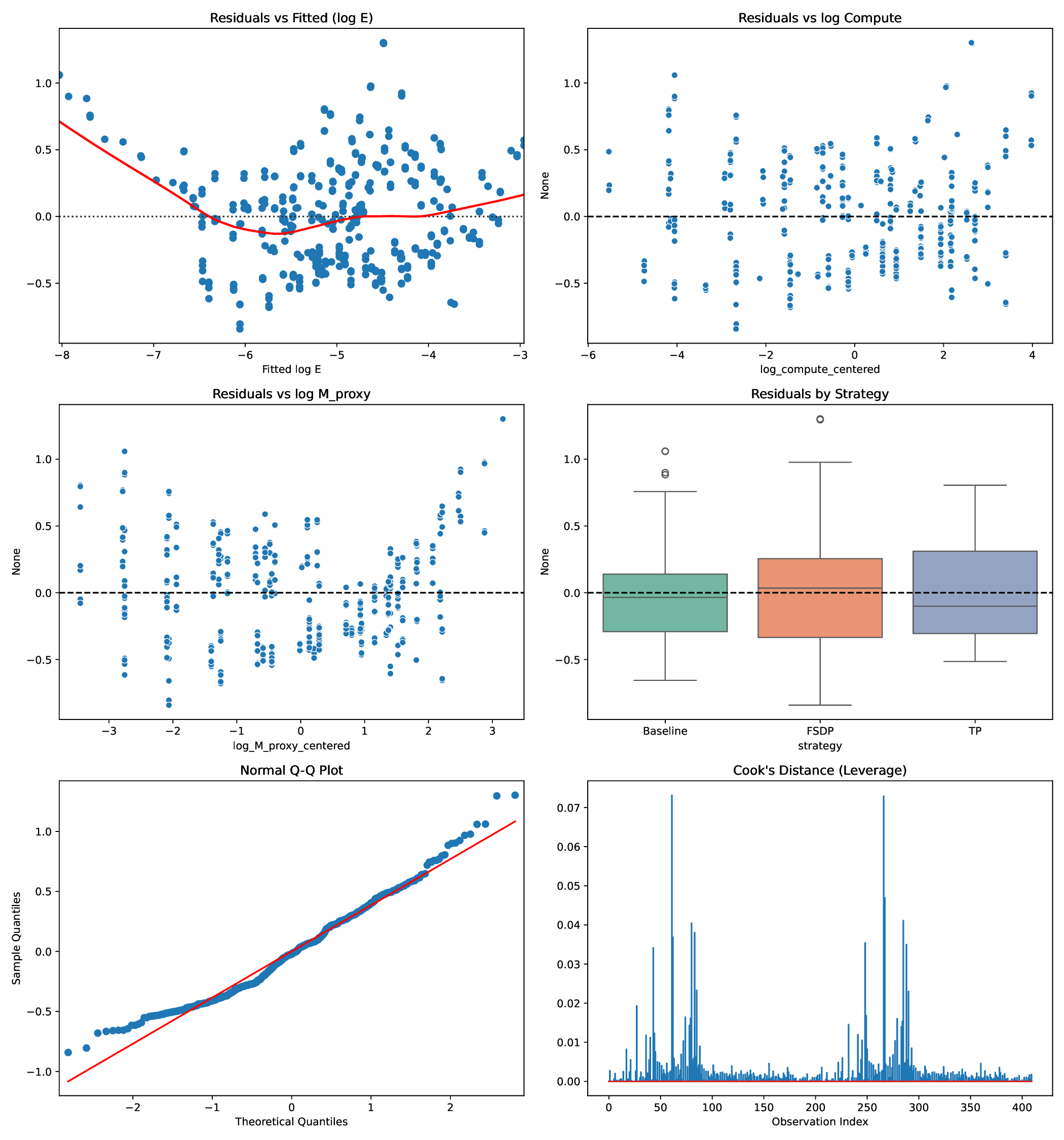}
\caption{Comprehensive diagnostic suite for the energy scaling law model ($N=410$). \textbf{Top Row:} Residuals vs. Fitted and Centered Log Compute show overall linearity and consistency. \textbf{Middle Row:} Residuals vs. Architectural Proxy ($\log M_{proxy}$) and the Strategy boxplot verify stability across model types and frameworks. \textbf{Bottom Row:} The Q-Q plot and Cook's Distance confirm that while the data is heavy-tailed due to hardware power spikes, no individual observations exert undue influence on the scaling coefficients.}
\label{fig:residuals}
\end{figure}

\section{Energy Estimation for State-of-the-Art Models}
\label{app:sota-energy}

Table~\ref{tab:sota} reports a comparison between energy consumption
figures reported in the literature for several state-of-the-art language models
and energy estimates obtained using our proposed framework. The models considered
include GPT-3, GShard, Meena, Switch Transformer, and T5, covering a wide range of
parameter counts, architectural depths, and training regimes.

\paragraph{Reconstruction methodology.}
For each model, we reconstruct the quantities required by our energy model from
publicly available information, including parameter count, number of layers $L$,
hidden width $W$, and reported training compute. To enable a controlled comparison
across heterogeneous systems, we evaluate all models under a \emph{fixed-data
regime}, holding the number of training tokens per epoch constant
($10^{14}$ tokens) and fixing the number of GPUs to $100$. From these quantities,
we derive the per-epoch compute $C_{\text{epoch}}$ and a memory-traffic proxy
$M_{\text{proxy}}$ based on activation volume. 

\paragraph{Estimated energy under parallel strategies.}
Using the reconstructed inputs, we estimate energy consumption per epoch and total
energy under two parallelization strategies: fully sharded data parallelism (FSDP)
and tensor parallelism (TP). Estimates are produced by applying the reduced-form
energy model with the corresponding speedup-based hardware-efficiency factors.
Reported energy values from the original publications are shown for reference.

\paragraph{Interpretation and limitations.}
Absolute energy values are not expected to match reported figures, as the original
models were trained on large-scale TPU or GPU clusters with different accelerators,
interconnects, and system-level overheads. Accordingly, the estimates in
Table~\ref{tab:sota} should be interpreted as
order-of-magnitude reconstructions rather than replications. Despite these
differences, the estimates preserve qualitative trends across models and
parallelization strategies. In particular, the results highlight the strong
sensitivity of energy consumption to parallel efficiency: for a fixed token
budget, TP consistently yields higher estimated energy than FSDP, reflecting the
dominant role of execution-time scaling captured by the hardware-efficiency term.

Overall, this analysis demonstrates how the proposed framework can be used to
extrapolate training energy for large models from limited experimental calibration,
providing a practical tool for comparative and predictive energy analysis when
direct measurements are unavailable.

\input{tab_params}

%% file: tab_params.tex
\begin{table}[ht]
\centering
\caption{Comparison of reported and estimated energy consumption for SOTA models.}
\label{tab:sota}
\scriptsize 
\setlength{\tabcolsep}{3pt} 
\begin{tabular}{llrrrcrrccrrrrr}
\toprule
\textbf{Model} & \textbf{Dev.} & \textbf{Params} & \textbf{FLOPs} & \textbf{L} & \textbf{W} & \textbf{GPUs} & \textbf{Tokens} & \textbf{C\_epoch} & \textbf{M\_proxy} & \textbf{Rep. Energy} & \multicolumn{2}{c}{\textbf{FSDP Est. (MWh)}} & \multicolumn{2}{c}{\textbf{TP Est. (MWh)}} \\
 & & (B) & (ZF) & & & & /Epoch & (FLOPs) & (Proxy) & (MWh) & Epoch. & Tot. & Epoch. & Tot. \\
\midrule
GPT-3  \cite{brown2020language}& OpenAI & 175.0 & 314.0 & 96 & 12288 & 100 & 1e14 & 1.81e25 & 2.42e23 & 1287.0 & 0.12 & 17.33 & 4.88 & 696.13 \\
GShard \cite{lepikhin2020gshard} & Google & 619.0 & 133.3 & 48 & 2048 & 100 & 1e14 & 2.57e23 & 5.03e21 & 24.1 & 0.04 & 1.72 & 1.45 & 69.08 \\
Meena \cite{adiwardana2020towards}& Google & 2.6 & 112.0 & 24 & 1024 & 100 & 1e14 & 3.21e22 & 6.29e20 & 232.0 & 0.02 & 6.90 & 0.79 & 276.94 \\
Switch \cite{fedus2022switch}& Google & 1500.0 & 82.2 & 32 & 2048 & 100 & 1e14 & 1.71e23 & 3.36e21 & 179.0 & 0.03 & 8.40 & 1.29 & 337.34 \\
T5  \cite{raffel2020exploring}& Google & 11.0 & 40.5 & 24 & 1024 & 100 & 1e14 & 3.40e22 & 1.26e21 & 85.7 & 0.02 & 2.65 & 0.85 & 106.47 \\
\bottomrule
\end{tabular}
\end{table}

%% file: proofs.tex
\newpage
\section{Proofs}
\subsection{Detailed Proof of Proposition~\ref{prop:roofline}}
\label{proof:roofline}
\begin{proposition}[Roofline energy lower bounds]
Let $C$ be the number of FLOPs executed in one training step, $M$ the number of
DRAM/HBM bytes transferred, $\mathrm{AI}\triangleq C/M$ the arithmetic intensity,
$F_{\mathrm{peak}}$ the device peak FLOP/s, $B$ the sustained memory bandwidth
(bytes/s), and let $P_{\mathrm{dyn}}$ denote the dynamic power during the active
portion of the step. Then the step time $T$ satisfies
\[
T \;\ge\; \frac{C}{\min\{F_{\mathrm{peak}},\,B\,\mathrm{AI}\}},
\]
and consequently the energy satisfies
\[
E \;\ge\; \frac{P_{\mathrm{dyn}}}{F_{\mathrm{peak}}}\,C
\qquad\text{and}\qquad
E \;\ge\; \frac{P_{\mathrm{dyn}}}{B}\,M .
\]
\end{proposition}

\begin{proof}
\emph{(Throughput bound.)}
The roofline model states that the achieved compute throughput is bounded by
\[
\mathrm{Perf} \;\le\; \min\{F_{\mathrm{peak}},\,B\,\mathrm{AI}\}
\;\triangleq\; F_{\mathrm{eff}} .
\]
The term $F_{\mathrm{peak}}$ is a hard upper bound imposed by the device’s
functional units and clock frequency, while the term $B\,\mathrm{AI}$ reflects the
maximum sustained compute rate achievable when memory bandwidth limits execution.

\emph{(Time lower bound.)}
By definition, $\mathrm{Perf}=C/T$, hence $C/T \le F_{\mathrm{eff}}$, which implies
\[
T \;\ge\; \frac{C}{F_{\mathrm{eff}}}
\;=\; \frac{C}{\min\{F_{\mathrm{peak}},\,B\,\mathrm{AI}\}} .
\]

\emph{(Energy lower bounds.)}
Energy is given by $E=\int_0^T P(t)\,dt$. During the active portion of the step, the
power draw is at least a positive dynamic level $P_{\mathrm{dyn}}$ (idle phases can
only increase $T$ for fixed $C$). Therefore,
\[
E \;\ge\; P_{\mathrm{dyn}}\,T
\;\ge\; P_{\mathrm{dyn}}\,\frac{C}{\min\{F_{\mathrm{peak}},\,B\,\mathrm{AI}\}} .
\]
Using $\min\{a,b\}\le a$ and $\min\{a,b\}\le b$ with
$a=F_{\mathrm{peak}}$ and $b=B\,\mathrm{AI}$ yields
\[
E \;\ge\; \frac{P_{\mathrm{dyn}}}{F_{\mathrm{peak}}}\,C
\quad\text{and}\quad
E \;\ge\; \frac{P_{\mathrm{dyn}}}{B\,\mathrm{AI}}\,C
\;=\; \frac{P_{\mathrm{dyn}}}{B}\,M .
\]
\end{proof}

\subsection{Detailed Proof of Theorem~\ref{thm:bands}}
\label{app: energy}
\begin{theorem}[Expected energy scaling bands]
Assume (i) quasi-stationary power during a training step; (ii) a stable operation
mix as hyperparameters vary; and (iii) over the hyperparameter ranges studied there
exist constants $k_1,k_2>0$ and an exponent $\beta\in[\tfrac12,1)$ such that the
dominant DRAM traffic co-scales with compute as
\[
k_1\,C^{\beta} \;\le\; M \;\le\; k_2\,C^{\beta}.
\]
Then, for device-specific constants
$\epsilon_{\mathrm{comp}},\epsilon_{\mathrm{mem}}>0$ and a fixed overhead
$E_0\ge 0$,
\[
\min\{\epsilon_{\mathrm{mem}}\,M,\;\epsilon_{\mathrm{comp}}\,C\}
\;\lesssim\; E
\;\lesssim\; \epsilon_{\mathrm{comp}}\,C + \epsilon_{\mathrm{mem}}\,M + E_0 .
\]
Moreover, over any range where $E_0$ is negligible, the effective log--log slope of
$E$ with respect to $C$ lies in $(\beta,1)$, yielding the reduced form
\[
E \;\approx\; \kappa\,C^{\alpha_C}\,N^{\alpha_N}\,\eta_h^{\alpha_\eta},
\quad\text{with}\quad
\alpha_C\in(\tfrac12,1],\;\alpha_\eta<0.
\]
\end{theorem}

\begin{proof}
\emph{(Sandwich bounds.)}
Let $a\triangleq\epsilon_{\mathrm{comp}}$ and $b\triangleq\epsilon_{\mathrm{mem}}$.
Using the upper co-scaling bound on $M$,
\[
E \;\le\; a\,C + b\,M + E_0
\;\le\; a\,C + b\,k_2\,C^{\beta} + E_0 .
\]
Using the lower co-scaling bound together with
Proposition~\ref{prop:roofline},
\[
E \;\ge\; \max\!\left\{
\frac{P_{\mathrm{dyn}}}{F_{\mathrm{peak}}}\,C,\;
\frac{P_{\mathrm{dyn}}}{B}\,M
\right\}
\;\ge\; \max\!\left\{a' C,\; b' C^{\beta}\right\},
\]
where $a'=\tfrac{P_{\mathrm{dyn}}}{F_{\mathrm{peak}}}$ and
$b'=\tfrac{P_{\mathrm{dyn}}}{B}k_1$ are positive constants. Hence, for $C$ in the
experimental range,
\begin{equation}
\label{eq:band-sandwich}
\max\{a' C,\, b' C^{\beta}\}
\;\le\; E
\;\le\; a\,C + b\,k_2\,C^{\beta} + E_0 .
\end{equation}

\emph{(Log--log slope bounds.)}
Consider $f(C)=a\,C + \tilde b\,C^{\beta}+E_0$ with
$\tilde b=b\,k_2>0$ and $0<\beta<1$. The instantaneous log--log slope is
\[
s(C)\;\triangleq\;\frac{d\log f(C)}{d\log C}
= \frac{C f'(C)}{f(C)}
= \frac{a\,C + \tilde b\,\beta\,C^{\beta}}
       {a\,C + \tilde b\,C^{\beta} + E_0}.
\]
All terms are positive, hence $s(C)\in(0,1)$. When $C$ is large enough that $E_0$ is
negligible, $s(C)$ lies strictly between $\beta$ and $1$, approaching $\beta$ in
the memory-dominated regime and $1$ in the compute-dominated regime. Therefore, any
log--log regression of $E$ on $C$ over a compact range with a non-degenerate mix
yields an effective exponent $\alpha_C\in(\beta,1)$.

\emph{(From $\beta$ to $(\tfrac12,1]$.)}
Architectural co-scalings determine $\beta$. For RNNs (LSTM/GRU),
$C=\Theta(B T h^2)$ while $M=\Theta(B T h)$ (dominant state and activation traffic),
implying $M=\Theta(\sqrt{C})$ and $\beta=\tfrac12$. For Transformers, compute grows
as $\Theta(BL^2 d)+\Theta(BL d^2)$ while $M=\Theta(BL d)$ (activations), yielding
$M=\Theta(C^{\beta})$ with $\beta\in(\tfrac12,1)$ over typical depth and width
sweeps. Consequently, $\alpha_C\in(\tfrac12,1]$ across these model families.

\emph{(Incorporating $N$ and $\eta_h$.)}
Dependence on model size $N$ and hardware efficiency $\eta_h$ introduces
multiplicative factors that vary slowly relative to $C$. In log space, these
contribute additive terms $\alpha_N\log N$ and $\alpha_\eta\log\eta_h$. Since higher
$\eta_h$ (greater utilization) reduces time and thus energy for fixed work, one has
$\alpha_\eta<0$, yielding the stated reduced form.
\end{proof}

\subsection{Detailed Proof of Lemma~\ref{lem:speedup_to_energy}}
\label{proof:speedup_to_energy}
\begin{lemma}[Transferability of speedup to energy efficiency]
Consider a fixed training workload executed on $N$ devices, and let $T(N)$ denote
the corresponding execution time. Let $P(t)$ be the instantaneous power draw and
$E(N)=\int_0^{T(N)} P(t)\,dt$ the total energy consumed. Assume that:

\begin{enumerate}
    \item[(A1)] The workload is fixed across device counts, i.e., the total amount
    of computation and memory traffic does not change with $N$.
    \item[(A2)] The average power draw
    $\bar{P}(N)\triangleq T(N)^{-1}\int_0^{T(N)}P(t)\,dt$
    is a slowly varying function of $N$ and is bounded above by a constant multiple
    of the single-device power draw.
    \item[(A3)] Changes in execution time dominate changes in average power as $N$
    varies.
\end{enumerate}

Then the energy scaling with respect to $N$ is primarily governed by the execution
time scaling, and the parallel speedup $S(N)=T(1)/T(N)$ induces a corresponding
hardware-efficiency factor
\[
\eta_h(N)\;\propto\;\frac{S(N)}{N}.
\]
In particular, any parametric model that accurately captures $T(N)$ (and hence
$S(N)$) provides a consistent proxy for the dependence of energy consumption on
parallelization.
\end{lemma}

\begin{proof}
Fix the workload and consider any device count $N\ge 1$. Define the average power
\[
\bar{P}(N)\;\triangleq\;\frac{1}{T(N)}\int_0^{T(N)} P(t)\,dt,
\]
so that, by the definition of an average,
\begin{equation}
\label{eq:energy_factorization}
E(N)\;=\;\int_0^{T(N)} P(t)\,dt \;=\; \bar{P}(N)\,T(N).
\end{equation}
Equation~\eqref{eq:energy_factorization} is an identity and holds without
assumptions.

\paragraph{Step 1: Relating energy ratios to time ratios.}
Consider the energy ratio between $N$ devices and a single device:
\begin{equation}
\label{eq:energy_ratio}
\frac{E(N)}{E(1)}
\;=\;
\frac{\bar{P}(N)\,T(N)}{\bar{P}(1)\,T(1)}
\;=\;
\frac{\bar{P}(N)}{\bar{P}(1)}\cdot\frac{T(N)}{T(1)}
\;=\;
\frac{\bar{P}(N)}{\bar{P}(1)}\cdot\frac{1}{S(N)}.
\end{equation}
Thus, for any $N$, energy scaling is exactly determined by the product of a power
ratio and an inverse speedup factor.

\paragraph{Step 2: Bounding and controlling the power term.}
Assumption (A2) states that $\bar{P}(N)$ is slowly varying in $N$ and is bounded
above by a constant multiple of $\bar{P}(1)$. Concretely, this implies the
existence of a constant $c\ge 1$ such that for all $N$ in the regime of interest,
\begin{equation}
\label{eq:power_bound}
1/c \;\le\; \frac{\bar{P}(N)}{\bar{P}(1)} \;\le\; c,
\end{equation}
possibly after restricting to a finite range of $N$ used in experiments. (The lower
bound is not essential for the conclusion, but it simplifies the interpretation:
average power cannot collapse to zero as $N$ grows under fixed workload execution.)

Combining \eqref{eq:energy_ratio} and \eqref{eq:power_bound} yields
\begin{equation}
\label{eq:energy_ratio_bounds}
\frac{1}{c}\cdot\frac{1}{S(N)}
\;\le\;
\frac{E(N)}{E(1)}
\;\le\;
c\cdot\frac{1}{S(N)}.
\end{equation}
Therefore, up to a multiplicative constant that does not scale rapidly with $N$,
energy is inversely proportional to the speedup.

\paragraph{Step 3: Formalizing ``time dominates power''.}
Assumption (A3) states that changes in execution time dominate changes in average
power as $N$ varies. One formal way to express this is via log-derivatives (elasticities):
\begin{equation}
\label{eq:dominance_assumption}
\left|\frac{d\log \bar{P}(N)}{d\log N}\right|
\;\ll\;
\left|\frac{d\log T(N)}{d\log N}\right|,
\end{equation}
over the device-count range considered. Using \eqref{eq:energy_factorization},
\[
\frac{d\log E(N)}{d\log N}
=
\frac{d\log \bar{P}(N)}{d\log N}
+
\frac{d\log T(N)}{d\log N}.
\]
Under \eqref{eq:dominance_assumption}, the first term is negligible relative to the
second, hence
\[
\frac{d\log E(N)}{d\log N}
\;\approx\;
\frac{d\log T(N)}{d\log N},
\]
which makes precise the statement that energy scaling is primarily governed by
time scaling.

Equivalently, combining (A2) and (A3) implies that, to leading order,
\begin{equation}
\label{eq:energy_time_proportional}
E(N)\;\approx\;\bar{P}(1)\,T(N)
\qquad\text{and thus}\qquad
E(N)\;\propto\;T(N),
\end{equation}
where the proportionality absorbs the slowly varying power factor.

\paragraph{Step 4: Induced hardware-efficiency factor.}
Define the \emph{ideal} (perfectly efficient) time on $N$ devices as
\[
T_{\mathrm{ideal}}(N)\;\triangleq\;\frac{T(1)}{N},
\]
corresponding to linear speedup and no parallel overhead. A standard notion of
parallel efficiency is then
\[
\eta_{\mathrm{par}}(N)
\;\triangleq\;
\frac{T_{\mathrm{ideal}}(N)}{T(N)}
=
\frac{T(1)/N}{T(N)}
=
\frac{S(N)}{N}.
\]
This quantity equals $1$ under perfect scaling and decreases when communication,
synchronization, or load imbalance reduce speedup.

Since \eqref{eq:energy_time_proportional} gives $E(N)\propto T(N)$ up to a slowly
varying factor, normalizing energy by the ideal scaling yields the same efficiency
structure. In particular, relative deviations from ideal energy scaling are
captured (up to constants from $\bar{P}(N)/\bar{P}(1)$) by $\eta_{\mathrm{par}}(N)$.
Thus we may define a device-count hardware-efficiency factor
\[
\eta_h(N)\;\propto\;\eta_{\mathrm{par}}(N)\;=\;\frac{S(N)}{N},
\]
where the proportionality constant can absorb the bounded power-ratio term in
\eqref{eq:energy_ratio}.

\paragraph{Step 5: Transferability from time models.}
Finally, let $\widehat{T}(N)$ be any parametric model that accurately predicts
execution time (or speedup) over the device counts considered. Then the induced
efficiency proxy
\[
\widehat{\eta}_h(N)\;\triangleq\;\frac{\widehat{S}(N)}{N},
\qquad\text{with}\qquad
\widehat{S}(N)=\frac{\widehat{T}(1)}{\widehat{T}(N)},
\]
provides a consistent approximation to $\eta_h(N)$ in the sense that it captures
the leading-order dependence of energy on $N$ via \eqref{eq:energy_time_proportional}.
Therefore, time-based benchmark results can be transferred into the energy model
through $\eta_h(N)$, completing the proof.
\end{proof}

%% file: modelConfigs.tex
\newpage
\section{Model Configurations}
\label{app:train-details}

In our experiments, we train Transformer models for a fixed number of five epochs
across a broad range of architectural and training hyperparameters in order to
systematically characterize their computational and energy behavior. We vary batch
size and key architectural parameters—including model depth, hidden dimensionality,
feedforward width, and the number of attention heads—to capture how design choices
affect training dynamics, hardware utilization, and energy consumption.

Fixing the number of training epochs across all configurations ensures a fair and
controlled comparison between models. This design isolates the effects of
architectural scaling and hyperparameter selection on energy efficiency, independent
of convergence speed or final task performance. As a result, differences in measured
energy consumption can be attributed primarily to changes in compute intensity,
memory traffic, and parallel utilization rather than optimization dynamics.

This controlled setup enables a systematic analysis of trade-offs between training
time, computational cost, and resource utilization. In particular, it allows us to
study how increases in model capacity and throughput translate into practical energy
costs under realistic training workloads. The resulting measurements provide a
comprehensive empirical basis for the energy models developed in the main text.

The hyperparameter ranges explored in our experiments are summarized in
Table~\ref{tab:experimental_parameters}.

\begin{table}[h!]
\centering
\caption{Transformer model hyperparameters and explored ranges.}
\label{tab:experimental_parameters}
\resizebox{\columnwidth}{!}{%
\begin{tabular}{|l|l|}
\hline
\textbf{Parameter} & \textbf{Range and purpose} \\ \hline
Batch size &
$64$ to $640$, to study the impact of throughput and data parallelism on training
energy \\ \hline
Number of layers ($\ell$) &
$2$ to $12$, to evaluate the effect of model depth on compute, memory traffic, and
energy scaling \\ \hline
Model dimension ($d_{\mathrm{model}}$) &
$128$ to $1280$, corresponding to the dimensionality of embeddings and internal
representations \\ \hline
Number of attention heads ($n_{\mathrm{heads}}$) &
$2$ to $12$, to analyze the energy implications of multi-head self-attention \\ \hline
Feedforward dimension ($d_{\mathrm{ff}}$) &
$512$ to $4096$, controlling the width of Transformer MLP blocks and their compute
and memory intensity \\ \hline
\end{tabular}
}
\end{table}